\documentclass{article}

 \usepackage[preprint]{neurips_2026}

\usepackage[utf8]{inputenc} 
\usepackage[T1]{fontenc}    
\usepackage{hyperref}       
\usepackage{url}            
\usepackage{booktabs}       
\usepackage{amsfonts}       
\usepackage{nicefrac}       
\usepackage{microtype}      
\usepackage{xcolor}         
\definecolor{ForestGreen}{RGB}{34,139,34}
\usepackage{amsmath}
\usepackage{amsthm,bbm,bm}
\usepackage{graphicx}
\usepackage{cleveref}
\usepackage{multirow}
\usepackage{subcaption}
\usepackage{wrapfig}
\usepackage{alltt}
\usepackage{algorithm}
\usepackage{algorithmic}

\newtheorem{theorem}{Theorem}

\DeclareMathOperator*{\argmax}{arg\,max}
\DeclareMathOperator*{\argmin}{arg\,min}
\newcommand{\ms}[1]{\tiny{$\pm$#1}}

\title{Post-hoc Selective Classification for Reliable Synthetic Image Detection}

\author{%
  Kaixiang Zheng\thanks{This work was done during an internship at Reality Defender.} \\
  University of Waterloo\\
  \texttt{k56zheng@uwaterloo.ca} \\
  \And
  Jacob H. Seidman \\
  Reality Defender \\
  \texttt{jacob@realitydefender.ai} \\
}

\begin{document}

\maketitle

\begin{abstract}
  As synthetic images become increasingly realistic, reliable synthetic image detection techniques are of pressing need to prevent their misuse. Despite satisfactory in-distribution performance, deep neural network-based synthetic image detectors (SIDs) lack reliability in deployment and often fail in the presence of common covariate shifts, resulting in poor detection accuracy. To avoid the risk caused by potential errors, we adopt a selective classification (SC) strategy by allowing SIDs to abstain from making low confidence predictions. For practicality, we focus on post-hoc methods which perform confidence estimation on a given SID without retraining. However, we show that conventional logit-based confidence score functions (CSFs) exhibit pathological behavior under covariate shifts, leading to SC performance close to or even worse than random guessing. To address this, we propose a simple yet effective SC framework for \textbf{Re}liable \textbf{S}ynthetic \textbf{I}mage \textbf{De}tection (ReSIDe). First, we generalize the notion of logits to an SID's intermediate layers from a centroid matching perspective, extending the use of logit-based CSFs to any layer of an SID. Then, we introduce a preference optimization algorithm that aggregates confidence scores extracted from different layers to a final confidence estimate by minimizing an upper bound of the area under the risk-coverage curve (AURC). Extensive experimental results show that ReSIDe significantly boosts the SC performance of various logit-based CSFs under common covariate shifts, achieving up to 69.55\% AURC reduction.
\end{abstract}

\section{Introduction} \label{sec:intro}

The rapid development of image generation models \citep{brock2018large, nichol2021glide, gu2022vector, rombach2022high, dhariwal2021diffusion, midjourney, wukong, sora2} has made synthetic images almost visually indistinguishable from real ones. To mitigate the risk caused by the misuse of realistic synthetic images, reliable synthetic image detectors (SIDs) are of urgent need. Although modern deep neural network-based (DNN-based) SIDs can be trained to attain near-perfect in-distribution (ID) accuracy, they are often unreliable in deployment upon real-world covariate shifts such as blurring, JPEG compression, low resolution, adversarial attacks, unseen semantic classes, images from unseen generators, etc.

To tackle this challenge and enhance the deployment reliability of SIDs, selective classification (SC) \citep{chow2003optimum, geifman2017selective} provides a principled solution. In an SC framework, a classifier is allowed to abstain from making predictions on samples with low confidence, aiming to increase the accuracy on selected samples. This creates a tradeoff between coverage (the proportion of selected samples) and selective risk (the risk on the selected samples), and an effective SC system should maximize the coverage while minimizing the selective risk. Existing SC methods can be categorized into two streams: methods that require retraining the classifier, and post-hoc SC methods that build upon a fixed classifier without retraining. In this work, we focus on the latter, as it's more cost-effective in practical scenarios. For post-hoc SC, since the classifier is pretrained and fixed, effective SC reduces to designing a confidence score function (CSF) that assigns higher confidence scores to correctly classified samples compared to wrongly classified ones. However, in Sec. \ref{sec:path_sc}, we observe that for synthetic image detection, conventional logit-based CSFs such as the maximum softmax probability (MSP) \citep{geifman2017selective} often become unreliable under common covariate shifts and can give performance at or worse than using a completely uninformed CSF.

To fix this pathological behavior of logit-based CSFs, we propose ReSIDe, a compact SC framework for \textbf{Re}liable \textbf{S}ynthetic \textbf{I}mage \textbf{De}tection with two key components. First, we generalize the notion of logits to an SID's intermediate layers as similarity scores to feature centroids via unsupervised linear probing (ULP). Concretely, for each intermediate layer, we use spherical k-means to compute the centroids of feature vectors from training samples, and use them to form a linear prober to get layerwise logits. Similar to SID's final logits, conventional logit-based CSFs can also be applied to these layerwise logits for confidence estimation, yielding complementary signals to the final layer's logits. Second, we linearly aggregate the confidence scores extracted from all layers, including that from the final logits, with weights optimized by a preference optimization algorithm on a hold-out validation set. We prove that minimizing the preference optimization objective, dubbed the ReSIDe loss, lowers an upper bound to the AURC, thus improving SC performance.

We evaluate ReSIDe on GenImage \citep{zhu2023genimage}, a large-scale SID benchmark, considering various covariate shifts. On both convolution-based and transformer-based SIDs, ReSIDe consistently and substantially improves SC performance of all kinds of logit-based CSFs, reducing AURC by up to 69.55\%. Importantly, ReSIDe is completely plug-and-play, successfully fixing the pathological behavior of final logit-based CSFs with simple linear probing and merely tens of trainable parameters.

The contributions of this work are summarized as follows:
\begin{itemize}
    \item We adopt selective classification to tackle the deployment reliability issue of SIDs.
    \item We identify a pathological behavior of final logit-based confidence estimation in synthetic image detection under common covariate shifts.
    \item We extend the use of logit-based CSFs to SID's intermediate layers based on unsupervised linear probing (ULP).
    \item We propose a preference optimization algorithm to aggregate layerwise confidence scores, provably minimizing an upper bound of the AURC.
    \item We evaluate our unified ReSIDe framework with extensive experiments, demonstrating significant and consistent AURC reductions across DNN architectures, covariate shifts, and CSFs, confirming ReSIDe as a practical post-hoc solution for reliable SID deployment.
\end{itemize}

\begin{figure}[!t]
    \centering
    \includegraphics[width=0.74\linewidth]{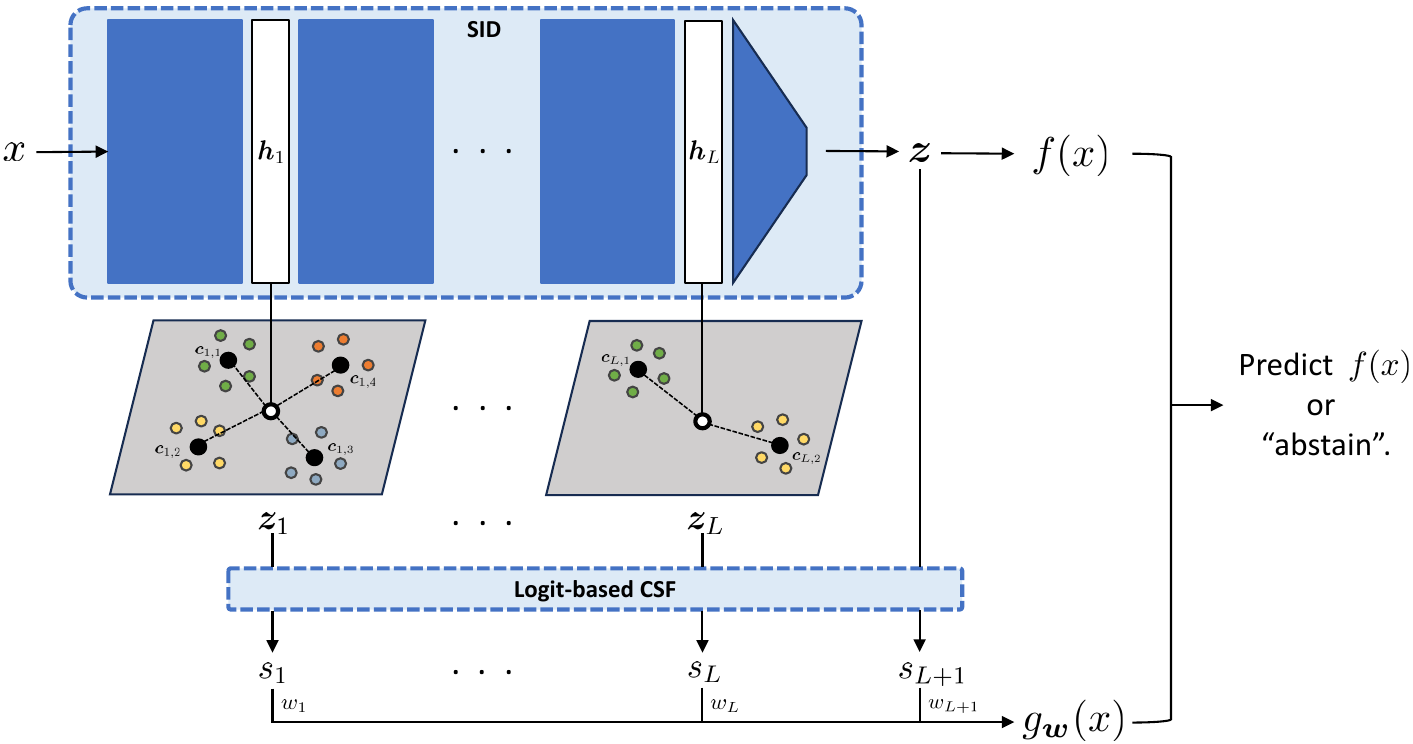}
    \caption{Illustration of the overall ReSIDe framework.}
    \label{fig:pipeline}
    \vspace{-5mm}
\end{figure}

\section{Background and related work} \label{sec:bg_rw}
A comprehensive review of related work on SIDs, SC, and CSFs is provided in Appendix \ref{app:rw}, where we carefully position our work among the relevant literature. In the following, we briefly discuss some background information to facilitate understanding of our proposed method.

\textbf{Logits as centroid matching scores.}
Although logits are typically seen as a linear projection of DNN's penultimate layer features, they have also been interpreted as similarity scores between a feature vector and feature centroids (also known as prototypes or templates). For example, \cite{muller2019does} view each (column) vector in DNN's last linear layer as a class template, and therefore each logit measures the dot product similarity between the penultimate layer feature vector and the corresponding class template. Neural collapse \citep{papyan2020prevalence} further strengthens this view by showing these weight vectors indeed converge to the class means of DNN's penultimate layer features, up to rescaling. In fact, this centroid matching principle is baked in the design of many DNN architectures. For example, in Prototypical Networks \citep{snell2017prototypical}, the centroids are given by the class means of learnable embeddings, and logits are computed by the negative squared Euclidean distance between a query's embedding and each class mean. In CLIP \citep{radford2021learning}, the centroids are given by the normalized text embeddings, and logits are computed by the cosine similarity between the normalized image embedding and each text embedding. In this work, inspired by the centroid matching interpretation of logits, we generalize the notion of logits to intermediate layers of SIDs using centroid-based linear probes, thus extending the use of logit-based CSFs from SID's final output to any layer of interest.

\textbf{Pairwise contrast for preference optimization.} In SC, a desirable CSF should give higher confidence scores to correctly classified samples than misclassified ones. To learn scalar functions with such explicit preference, the following probabilistic approach is commonly adopted. Consider a model $r_{\phi}: \mathbb{R}^d \to \mathbb{R}$ parameterized by $\phi$ that receives a \textit{d}-dimensional feature vector as input and outputs a real number. An input $a$ is preferred over another input $b$, denoted by $a \succ b$, if $r_{\phi}(a)>r_{\phi}(b)$. Modeling the probability of $a \succ b$ by the Bradley–Terry model \citep{bradley1952rank}, i.e., $\text{Pr}\{a \succ b\}=e^{r_{\phi}(a)}/(e^{r_{\phi}(a)}+e^{r_{\phi}(b)})$,
one can train $r_{\phi}$ by minimizing the expected negative log-likelihood (NLL) loss
\begin{small}\begin{align}
\min_{\phi} \mathbb{E}_{a \sim A, b \sim B} \left [-\ln \text{Pr}\{a \succ b\}\right]=\min_{\phi}\mathbb{E}_{a \sim A, b \sim B} \left [\ln(1 + \exp(r_{\phi}(b) - r_{\phi}(a)))\right], \label{eq:po_obj}
\end{align}\end{small}where $A$ and $B$ are the distributions of the preferred and non-preferred patterns. Such trained model tends to return a higher score for a future input if it's more preferable.

This general framework has wide applications in various fields. In information retrieval, RankNet \citep{burges2005learning} follows the framework to rank query-document pairs for building search engines; in large language model (LLM) post-training, reinforcement learning from human feedback (RLHF) \citep{ouyang2022training} and direct preference optimization (DPO) \citep{rafailov2023direct} follow the framework for reward modeling based on human preference; more recently, Mole-PAIR \citep{he2025can} also follows the framework to perform OOD detection for molecular foundation models. In this work, we adopt this framework to aggregate confidence scores extracted from different DNN layers for better SC performance.

\section{Notations and preliminaries}
\textbf{Selective classification.} Consider a classification problem with input space $\mathcal{X}$ and label space $\mathcal{Y} = [C] := \{1, \dots, C\}$. Let $P$ be an unknown data distribution over $\mathcal{X} \times \mathcal{Y}$. The risk of a classifier $f: \mathcal{X} \to \mathcal{Y}$ is defined as $R(f) = \mathbb{E}_{(x,y) \sim P}[l(f(x), y)]$, where $l$ is a loss function. By convention, we set $l$ to the 0/1 loss, $l(f(x), y) = \mathbbm{1}[f(x) \neq y]$, where $\mathbbm{1}[\cdot]$ denotes the indicator function. In this case, the risk is equal to the error rate of the classifier. A selective classifier is defined as a pair $(f,g)$, where $f$ is a classifier and $g:\mathcal{X}\to\mathbb{R}$ is a confidence score function (CSF) that measures the classifier's confidence on its prediction for a given input. Given a threshold $t$, the selective classifier accepts an input $x$ and outputs the corresponding prediction $f(x)$ if $g(x)\ge t$; otherwise, it abstains from making a prediction, i.e.,
\begin{small}\begin{equation}
    (f,g)(x) := \begin{cases} f(x) & \text{if}~~g(x) \geq t, \\
                           \text{``abstain''} & \text{otherwise.}
             \end{cases}
\end{equation}\end{small}The performance of a selective classifier is characterized by its coverage and selective risk. Given $t$, the coverage of a selective classifier is defined as the probability of accepting an input, i.e., $\phi_t(f,g)=\text{Pr}\{g(x)\ge t\}$, and the selective risk is defined as the conditional risk given acceptance, i.e., $R_t(f,g)=\mathbb{E}_P[l(f(x),y)\mid g(x)\ge t]$. For 0/1 loss, the selective risk is equal to the conditional error rate given acceptance, i.e., $R_t(f,g)=\text{Pr}\{f(x) \neq y \mid g(x)\ge t\}$. Note that a classifier's risk is equal to its selective risk at full coverage, i.e., $R(f)=R_t(f,g)$ if $t$ is chosen such that $\phi_t(f,g)=1$. Empirically, given a test set $T=\{(x_i,y_i)\}_{i=1}^N$ drawn i.i.d from $P$, we evaluate the empirical coverage and the empirical selective risk as follows:
\begin{small}\begin{equation}
\hat{\phi}_t(f,g,T)=\frac{1}{N}\sum_{i=1}^N\mathbbm{1}[g(x_i)\ge t],~\hat{R}_t(f,g,T)=\frac{\sum_{i=1}^{N} l(f(x_i),y_i)\mathbbm{1}[g(x_i)\ge t]}{\sum_{i=1}^{N}\mathbbm{1}[g(x_i)\ge t]}.
\end{equation}\end{small}Given a decent CSF, one can typically trade coverage for selective risk by increasing the threshold $t$ and abstaining from more uncertain samples. This tradeoff can be visualized by the risk-coverage (RC) curve, which plots $\hat{R}_t(f,g,T)$ as a function of $\hat{\phi}_t(f,g,T)$. A lower RC curve indicates better SC performance, since for a fixed coverage a smaller selective risk is preferred and for a fixed selective risk a larger coverage is preferred. To summarize the RC curve using a single scalar, an established metric is the \textit{area under the risk-coverage curve} (AURC), denoted by $\text{AURC}(f,g,T)$. Since this paper focuses exclusively on post-hoc SC, we suppress the notational dependence on the fixed classifier $f$. In this case, the quality of the CSF is pivotal to the SC performance, since the RC curve and the resulting $\text{AURC}(g,T)$ are directly determined by $g$, given the test set $T$.

\textbf{Two special cases of CSF and the corresponding SC performance.} First, if a CSF $g^*$ enables perfect ranking of samples in $T$ such that correctly classified samples always get higher confidence scores than those of wrongly classified ones, i.e., $g^*(x^+)>g^*(x^-)$ for any $x^+\in\{x\in T \mid f(x) = y\}$ and $x^-\in\{x\in T \mid f(x) \neq y\}$, then we achieve the optimal SC performance on $T$ with the lowest possible AURC. Second, if a CSF is independent to the classifier's error, e.g., a random CSF, its AURC is equal to the classifier's risk, which can be approximated on $T$ by the empirical error rate $\hat{r}=\frac{1}{N}\sum_{i=1}^N l(f(x_i),y_i)=\frac{1}{N}\sum_{i=1}^N\mathbbm{1}[f(x_i) \neq y_i]$. CSFs used in practice typically yield AURCs between $\text{AURC}(g^*,T)$ and $\hat{r}$. However, if a CSF generally gives higher confidence scores to wrongly classified samples than correctly classified ones, the AURC can be larger than $\hat{r}$.

\textbf{Logit-based CSFs}. DNN-based classifiers generally make predictions by $f(x)=\hat{y}:=\argmax_{c \in \mathcal{Y}} z_c$, where $\bm{z}(x) = [z_1, \dots, z_C] \in \mathbb{R}^C$ is referred to as logits. In this case, the most popular CSFs are logit-based CSFs, which compute confidence scores directly from the logits $\bm{z}$. Since softmax is routinely applied to logits to obtain an estimate of the posterior distribution $[P(y|x)]_{y\in \mathcal{Y}}$, the maximum probability mass output by softmax can be naturally regarded as the confidence on the prediction $\hat{y}$. This yields a logit-based CSF called the maximum softmax probability (MSP) \citep{hendrycks2016baseline, geifman2017selective}. Beyond MSP, other typical logit-based CSFs include softmax margin (SM) \citep{belghazi2021classifiers}, the max logit (ML) \citep{basart2022scaling}, the logits margin (LM) \citep{le1990handwritten, liang2024selective}, the negative entropy (NE) \citep{belghazi2021classifiers}, and the negative Gini index (NGI) \citep{granese2021doctor}. Formally, they are defined as:

\begin{table}[h!]
\vspace{-2mm}
\centering
\resizebox{\textwidth}{!}{
\begin{tabular}{lll}
$\text{MSP}(\bm{z}) = \sigma_{\hat{y}}(\bm{z})$, & $\text{SM}(\bm{z}) = \sigma_{\hat{y}}(\bm{z}) - \max_{c \in \mathcal{Y}:~c \neq \hat{y}} \sigma_c(\bm{z})$, & $\text{ML}(\bm{z}) = z_{\hat{y}}$, \\
$\text{LM}(\bm{z}) = z_{\hat{y}} - \max_{c \in \mathcal{Y}:~c \neq \hat{y}} z_c$, & $\text{NE}(\bm{z}) = \sum_{c \in \mathcal{Y}} \sigma_c(\bm{z}) \ln \sigma_c(\bm{z})$, & $\text{NGI}(\bm{z}) = -1 + \sum_{c \in \mathcal{Y}} \sigma_c(\bm{z})^2$,
\end{tabular}}
\vspace{-3mm}
\end{table}

where $\sigma(\cdot)$ is the softmax operation, and $\sigma_c(\cdot)$ is the \textit{c}-th element of the softmax-processed vector.

In this paper, the classifier $f$ is a DNN-based SID, $C=2$ for the binary SID task, and the data distribution $P$ denotes the training distribution of $f$ with potential covariate shifts.

\section{Empirical evaluation of logit-based CSFs} \label{sec:eval_csf}
In this section, we evaluate the SC performance of standard logit-based CSFs for synthetic image detection with six types of common covariate shifts. We first detail the data and model setup and then demonstrate pathological SC performance with these logit-based CSFs, thus motivating our method.

\subsection{Data and model preparation}
\textbf{Dataset.} GenImage \citep{zhu2023genimage} is a large-scale dataset for synthetic image detection, consisting of over 2.6 million images. Roughly half of the images are real and the other half are synthetic from eight generators, namely BigGAN \citep{brock2018large}, GLIDE \citep{nichol2021glide}, VQDM \citep{gu2022vector}, Stable Diffusion V1.4 \citep{rombach2022high}, Stable Diffusion V1.5 \citep{rombach2022high}, ADM \citep{dhariwal2021diffusion}, Midjourney \citep{midjourney}, and Wukong \citep{wukong}. Note that GenImage is built upon ImageNet-1K \citep{deng2009imagenet} and covers images with the same 1000 semantic classes. We curate the datasets used in our experiments from GenImage to create representative covariate shift scenarios. The training set consists of real images and images generated by ADM and BigGAN, covering 50 semantic classes. To test SID's performance under different distributions, we curate seven test sets. The in-distribution (ID) test set follows the same distribution as the training set, i.e., containing images from the same sources and semantic classes. The remaining six test sets contain various covariate shifts:
(1) \textbf{Blur:} images are processed by Gaussian blurring with $\sigma \sim \text{Uniform}[3,5]$, where $\sigma$ is the standard deviation (std) of the Gaussian kernel;
(2) \textbf{Unseen semantic classes (Unseen Cls.):} images from 950 semantic classes not presented in training;
(3) \textbf{Low resolution (Low Res.):} images are resized while preserving the aspect ratio, such that the length of the shorter edge is uniformly random in $[64,96]$;
(4) \textbf{FGSM:} images are perturbed using the fast gradient sign method (FGSM) \citep{goodfellow2014explaining}, with the distortion budget $\epsilon$ uniformly random in $\{1/255, 2/255, ..., 8/255\}$;
(5) \textbf{JPEG:} images are compressed with JPEG with $q \sim \text{Uniform}[30,70]$, where $q$ is the quality factor of JPEG compression;
(6) \textbf{Unseen generators (Unseen Gen.):} synthetic images are generated by Wukong, Midjourney, Stable Diffusion V1.4, and Stable Diffusion V1.5, which are not presented in training.
Note that to isolate the effect of different covariate shifts and study them independently, each covariate-shifted testing distribution differs from the training distribution in only one aspect, while other aspects stay the same as the training distribution. For example, for the dataset with unseen semantic classes, the synthetic images are generated by ADM and BigGAN; for the dataset with unseen generators, the semantic classes are the same 50 classes as in the training set. 

\textbf{Models.} We train two SIDs using representative DNN architectures, a CNN-based architecture ResNet50 \citep{he2016deep} and a transformer-based architecture Swin-T \citep{liu2021swin}. Both models are trained on the training set using standard supervised learning and kept fixed during all SC evaluations. Note that the test set with FGSM perturbations is generated by attacking the ResNet50-based SID.

Tab.~\ref{tab:sum_mode} reports the testing error rates of the two SIDs under different data distributions. As shown, both models achieve error rates close to zero on ID data but suffer from substantial performance degradation under most covariate shifts, often approaching random guessing level with error rates close to 0.5. The covariate shift based on unseen semantic classes is the only exception, where the accuracy degradation is below 5\% for both models.

\begin{table}[!ht]
\vspace{-3mm}
\centering
\caption{Testing error rates of SIDs under different data distributions.}
\label{tab:sum_mode}
\resizebox{0.7\textwidth}{!}{
\begin{tabular}{cccccccc}
\toprule
& ID & Blur & Unseen Cls. & Low Res. & FGSM & JPEG & Unseen Gen.\\
\midrule
ResNet50 & 0.0242 & 0.4950 & 0.0732 & 0.5000 & 0.4767 & 0.4783 & 0.4850 \\
Swin-T & 0.0150 & 0.4950 & 0.0341 & 0.4875 & 0.3975 & 0.4792 & 0.4669 \\
\bottomrule
\end{tabular}}
\vspace{-3mm}
\end{table}

\subsection{Pathological SC performance} \label{sec:path_sc}
\begin{figure}[!t]
\centering
\begin{subfigure}[b]{0.22\linewidth}  
\centering 
\includegraphics[width=\linewidth]{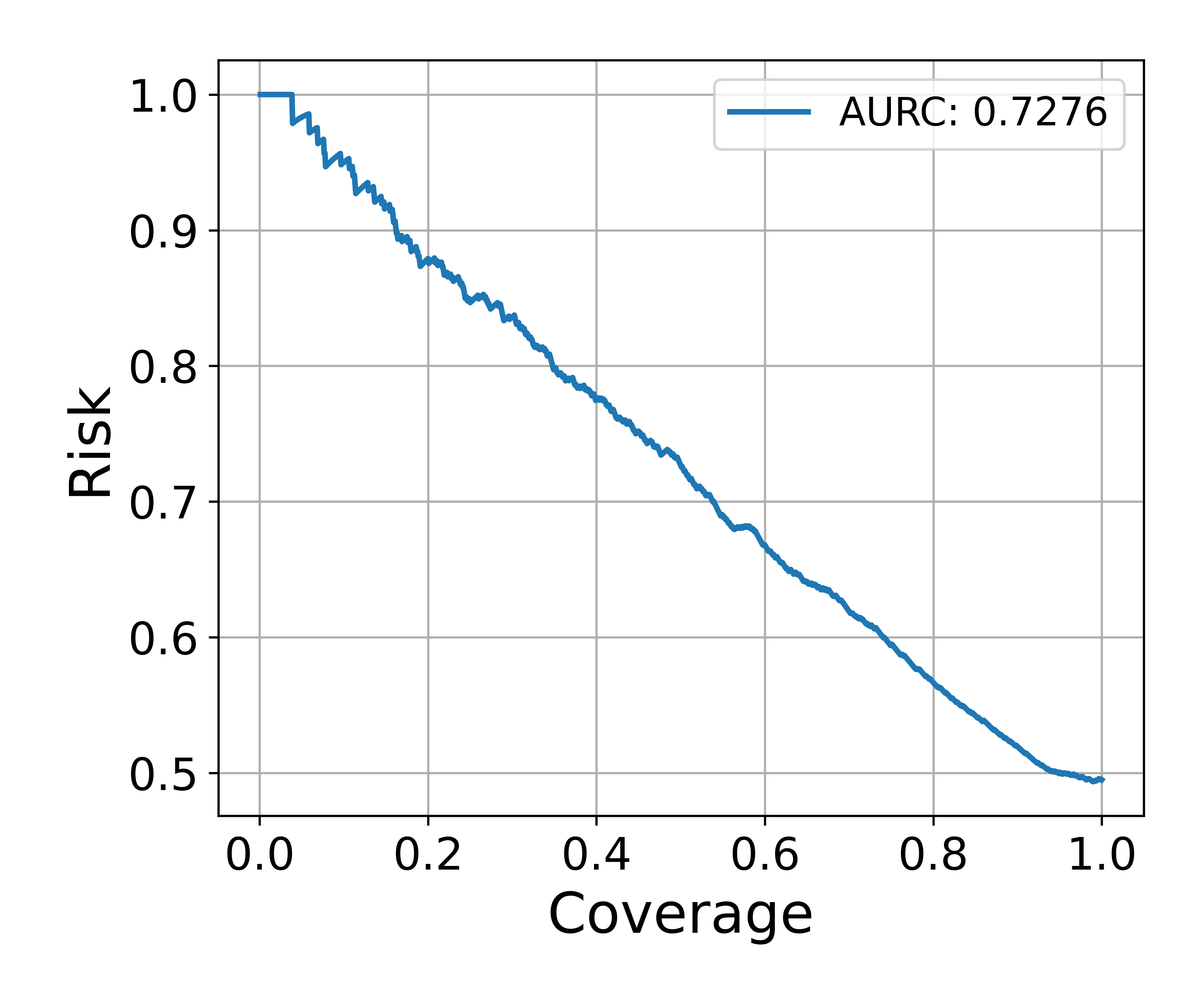}
\caption*{(a) Blur}
\end{subfigure}
\begin{subfigure}[b]{0.22\linewidth}
\centering 
\includegraphics[width=\linewidth]{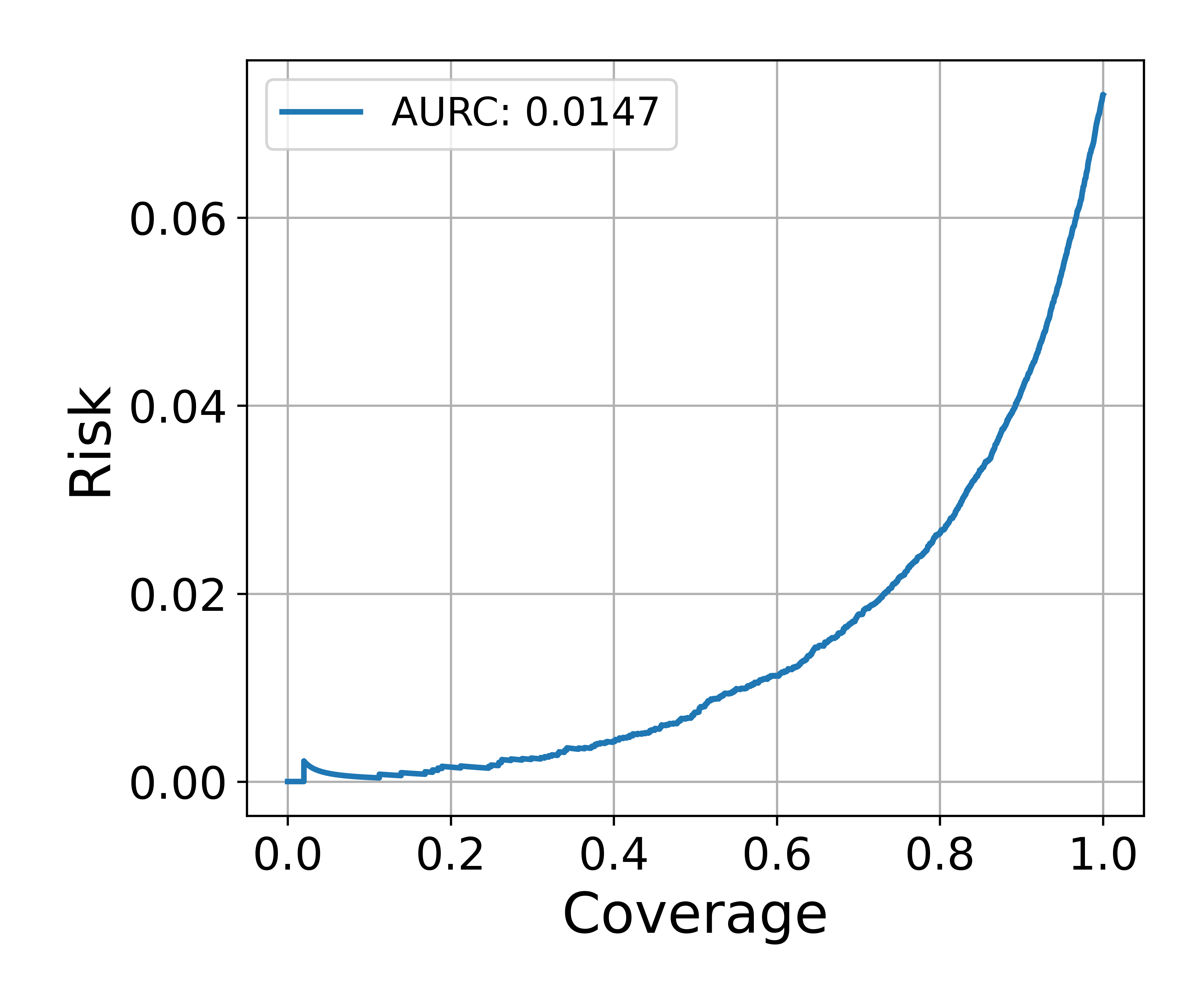}
\caption*{(b) Unseen Cls.}
\end{subfigure}
\begin{subfigure}[b]{0.22\linewidth}  
\centering 
\includegraphics[width=\linewidth]{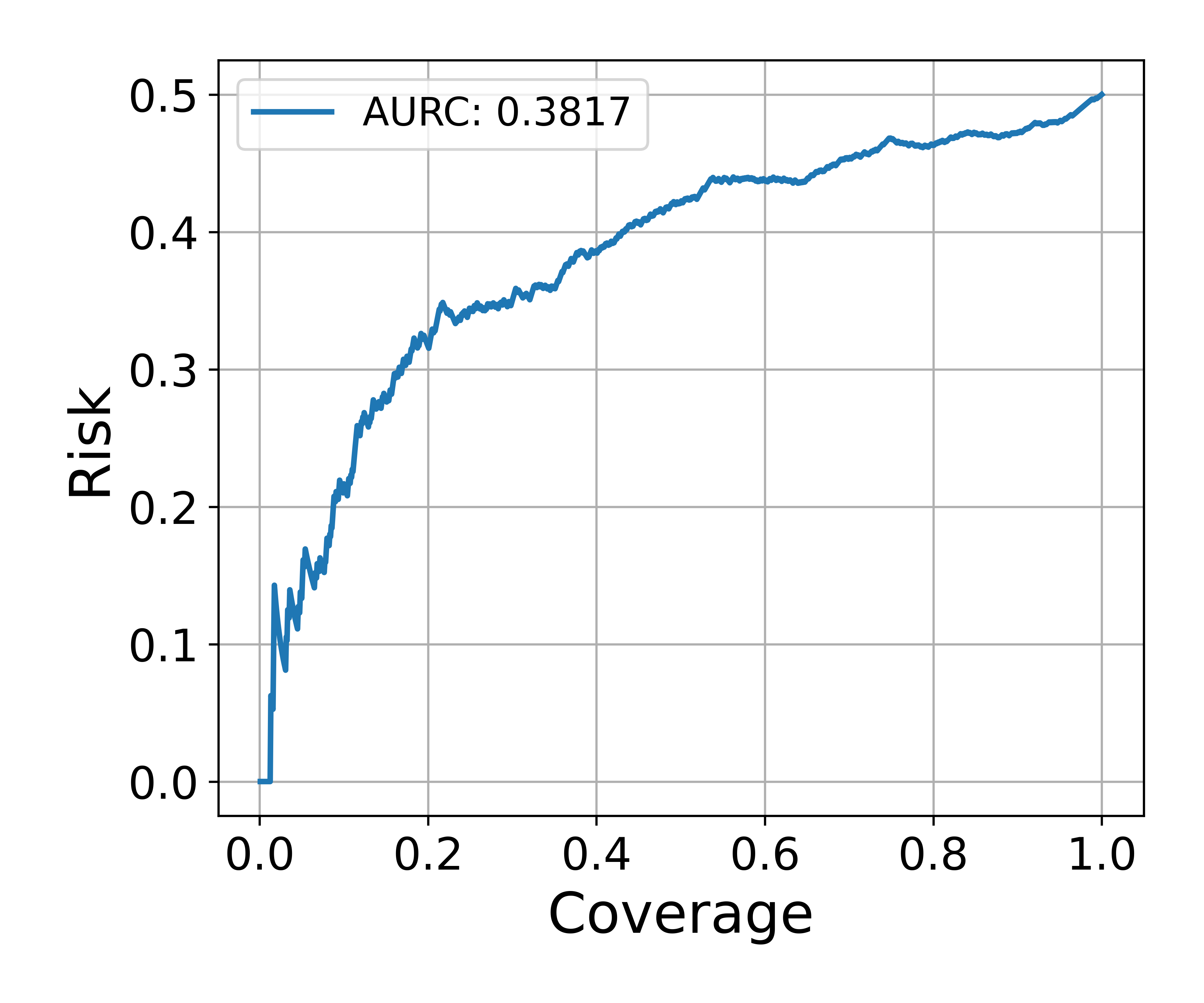}
\caption*{(c) Low Res.}
\end{subfigure}
\begin{subfigure}[b]{0.22\linewidth}
\centering 
\includegraphics[width=\linewidth]{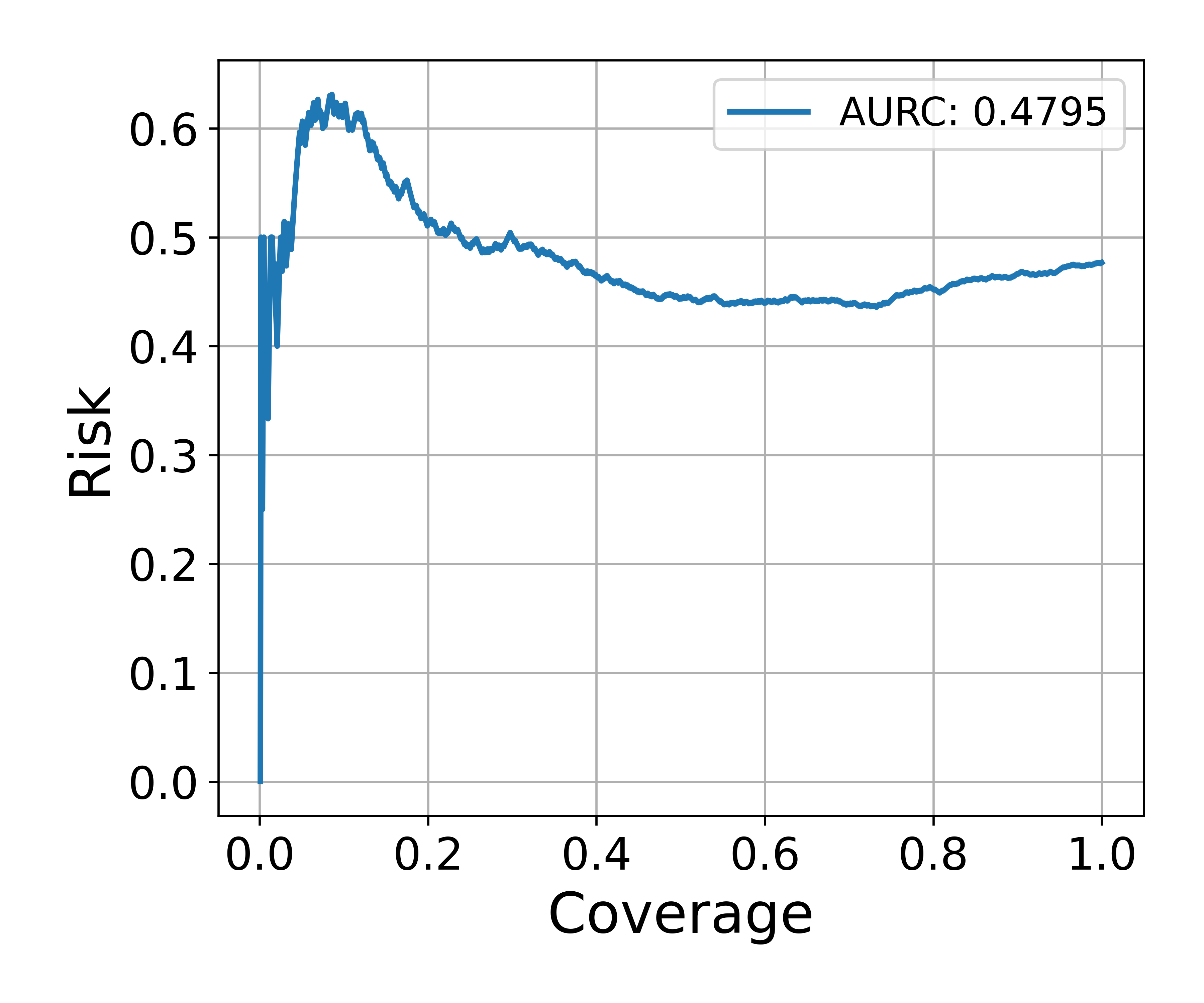}
\caption*{(d) FGSM}
\end{subfigure}
\begin{subfigure}[b]{0.22\linewidth}  
\centering 
\includegraphics[width=\linewidth]{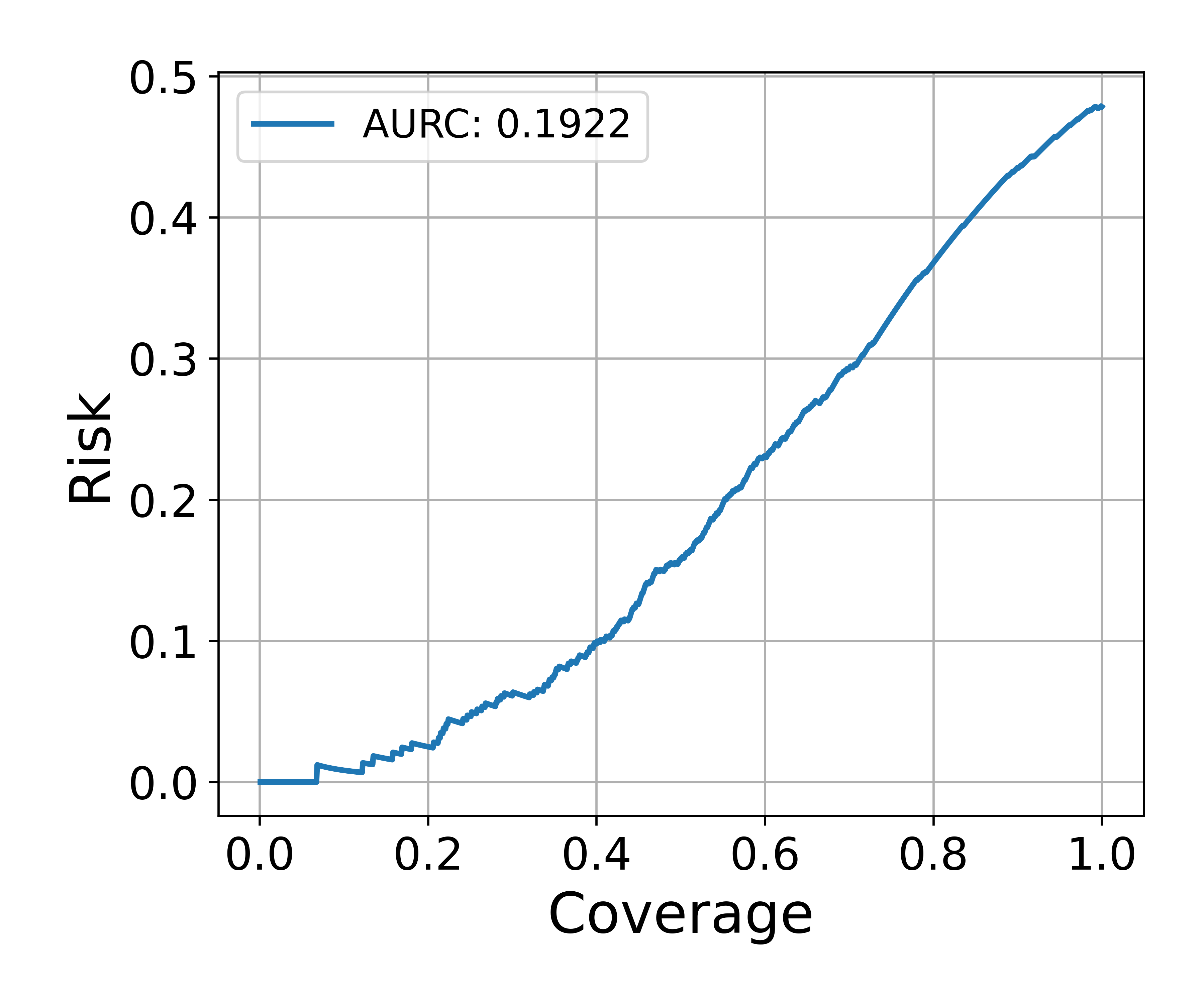}
\caption*{(e) JPEG}
\end{subfigure}
\begin{subfigure}[b]{0.22\linewidth}
\centering 
\includegraphics[width=\linewidth]{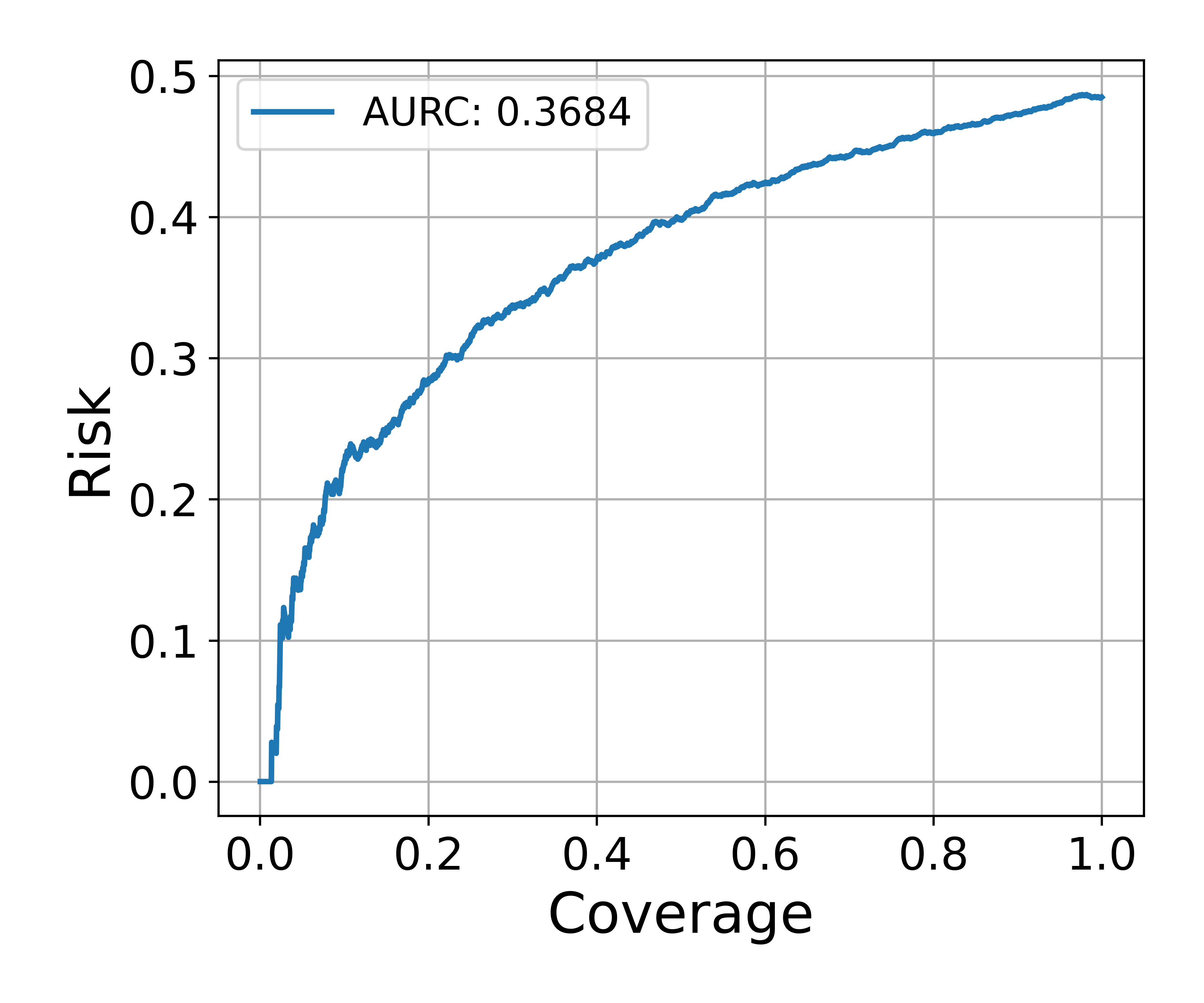}
\caption*{(f) Unseen Gen.}
\end{subfigure}
\begin{subfigure}[b]{0.22\linewidth}  
\centering 
\includegraphics[width=\linewidth]{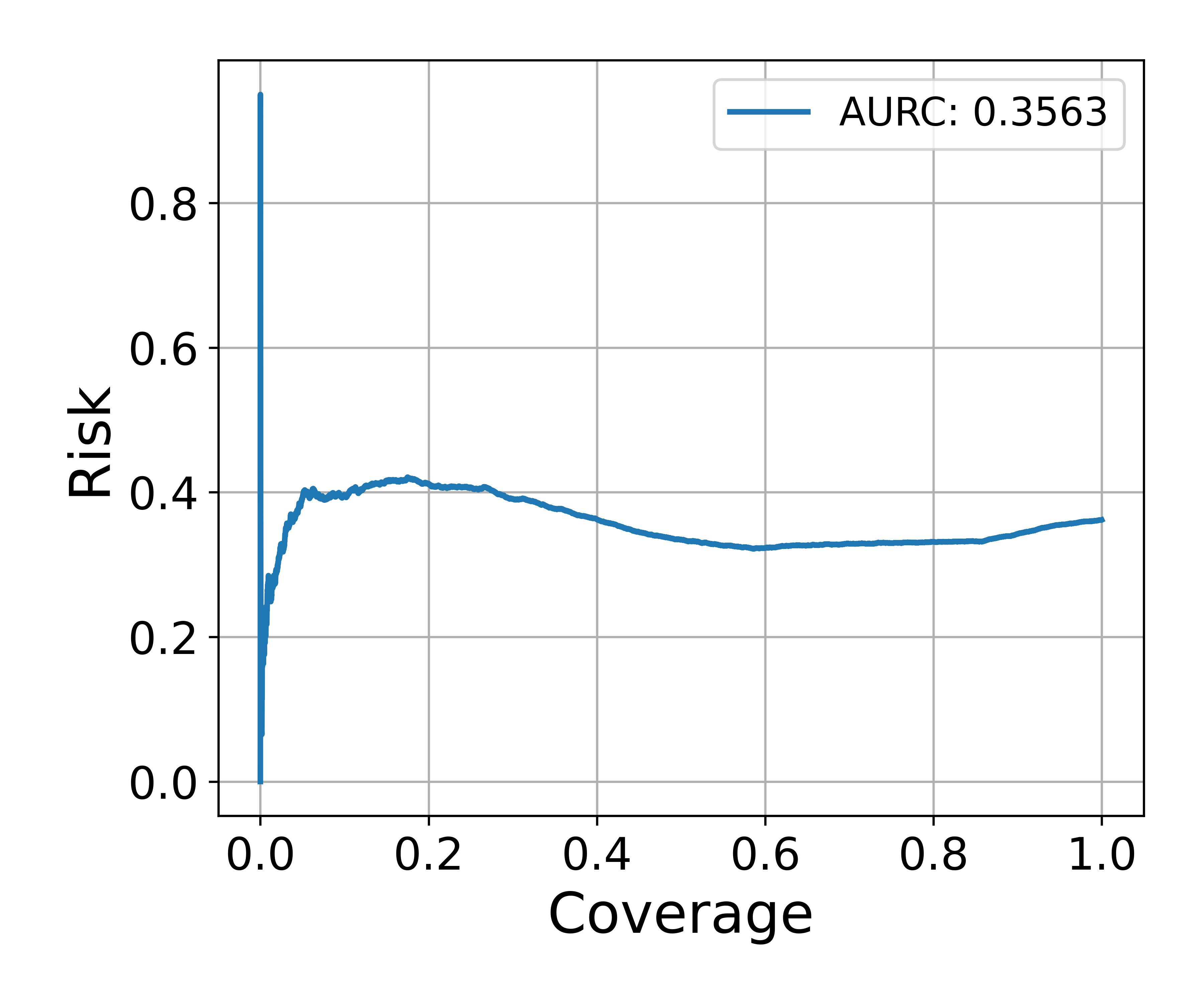}
\caption*{(g) ResNet50}
\end{subfigure}
\begin{subfigure}[b]{0.22\linewidth}  
\centering 
\includegraphics[width=\linewidth]{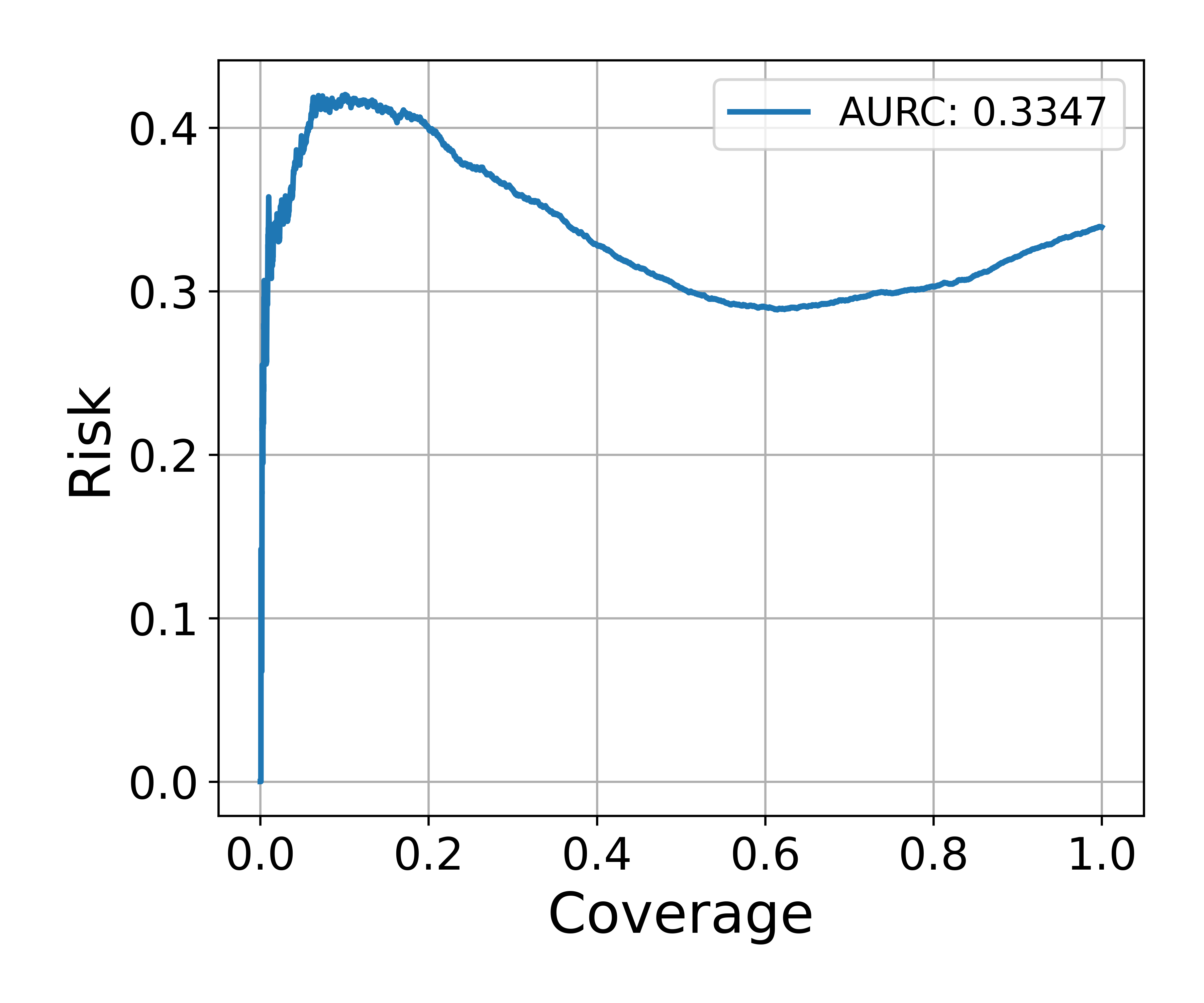}
\caption*{(h) Swin-T}
\end{subfigure}
\caption{(a)-(f): RC curves for MSP evaluated using ResNet50 on datasets with different covariate shifts. (g)-(h): RC curves for MSP under the mixture distribution, evaluated using ResNet50 and Swin-T, respectively.} 
\label{fig:rc}
\vspace{-5mm}
\end{figure}

We now examine the SC performance of logit-based CSFs under different data distributions. Figs.~\ref{fig:rc}(a)-(f) show the RC curves for MSP evaluated using ResNet50 on six covariate-shifted test sets, together with the resulting AURC. While the SC performance on Unseen Cls. and JPEG are decent, those on the remaining cases are much poorer. Specifically, the AURCs on Blur and FGSM are greater than the empirical error rates on respective test sets, meaning that MSP is behaving worse than a random CSF. A pathological SC performance is especially clear in the case of Blur, where the selective risk is going up when the coverage goes down, indicating that correctly classified samples generally have lower MSP scores than wrongly classified ones and are being rejected first when the threshold $t$ increases. Similar pathological performance is demonstrated for all six logit-based CSFs and both models, whose AURCs are presented in Tab.~\ref{tab:aurc_six}.

In practical deployment scenarios, an input could be in-distribution (ID) or from various covariate-shifted distributions. It's therefore of interest to analyze the SC performance on a mixture of multiple data distributions. We consider an even mixture of seven distributions\footnote{The number of samples varies among test sets. To handle the imbalance and create an even mixture, we reweight samples when computing SC metrics. Please refer to Appendix \ref{app:reweight} for details.}: ID, Blur, Unseen Cls., Low Res., FGSM, JPEG, Unseen Gen. are assumed to happen equal likely. In this case, the error rate of ResNet50 is 0.3618, and that of Swin-T is 0.3393 (the average of error rates over all seven cases). Figs.~\ref{fig:rc}(g)-(h) plot the RC curves for MSP evaluated using both models under the mixture distribution, together with the resulting AURC. As shown, the selective risk hardly goes down as the coverage decreases, resulting in AURCs close to those of random CSFs on both models. Other logit-based CSFs yield similar SC performance, as shown in Tab.~\ref{tab:aurc_combined}. These findings reveal that conventional final logit-based CSFs are not reliable for SID selective classification under covariate shifts. Therefore, we are motivated to exploit additional information from SID's intermediate layers.

\section{Methodology} \label{sec:method}
\subsection{Generalizing logit-based CSFs}
As mentioned in Sec.~\ref{sec:bg_rw}, logits can be regarded as similarity scores between a feature vector and feature centroids. Based on this principle, we generalize the notion of logits to SID's intermediate layers as follows. Given a DNN-based SID $f$ with $L$ hidden layers, for each layer $l \in [L]$, define a set of $K_l$ feature centroids as $\{\bm{c}_{l,1}, ..., \bm{c}_{l,K_l}\}$. Then, the generalized logits at layer $l$ can be computed by
\begin{small}\begin{equation}
    \bm{z}_l = \left [ z_{l,k} \right ]_{k=1}^{K_l} = \left [ s(\bm{h}_l, \bm{c}_{l,k}) \right ]_{k=1}^{K_l}, \label{eq:logits_general}
\end{equation}\end{small}where $\bm{h}_l$ is a feature vector at layer $l$ corresponding to some input $x$, and $s(\cdot,\cdot)$ is a given similarity metric such as cosine similarity, negative Euclidean distance, negative Mahalanobis distance, etc. Given $\bm{z}_l$ and any logit-based CSF, we can compute a confidence score $s_l$ at layer $l$. Defining $\hat{y}_l:=\argmax_{k \in [K_l]} z_{l,k}$, we have, for example, $s_l=\text{MSP}(\bm{z}_l) = \sigma_{\hat{y}_l}(\bm{z}_l)$, with MSP being the underlying logit-based CSF. Thus, given a logit-based CSF, we can extract in total $L+1$ confidence scores $\{s_1, ..., s_L, s_{L+1}\}$ from an SID, where the first $L$ scores are extracted from $L$ hidden layers, respectively, and the last score $s_{L+1}$ is the standard score computed from SID's final logits $\bm{z}$.

\textbf{Centroids and similarity metric.} In this work, we construct feature centroids by clustering. For each layer $l$, we obtain feature vectors $\bm{h}_l$ of all samples in SID's training set. Then, we perform spherical k-means \citep{dhillon2001concept} to get $K_l$ feature centroids, where $K_l$ is determined by silhouette analysis \citep{shahapure2020cluster}. As for the similarity metric $s(\cdot, \cdot)$, we adopt cosine similarity to be consistent with the clustering objective of spherical k-means. Given these specific design choices, we can rewrite Eq.~(\ref{eq:logits_general}) as
\begin{small}\begin{equation}
    \bm{z}_l = C_l^\top\hat{\bm{h}}_l, \label{eq:ulp}
\end{equation}\end{small}where $C_l=[\bm{c}_{l,1}, ..., \bm{c}_{l,K_l}]$ is a matrix composed of feature centroids obtained from spherical k-means, and $\hat{\bm{h}}_l=\bm{h}_l/||\bm{h}_l||_2$ is a $l_2$-normalized feature vector. Note that the centroids given by spherical k-means are on the unit sphere, so no normalization is needed for them. Since the feature centroids are obtained by unsupervised clustering, and Eq.~(\ref{eq:ulp}) is a simple linear projection of SID's intermediate features, we dub the logits computation in Eq.~(\ref{eq:ulp}) unsupervised linear probing (ULP).

\subsection{Preference optimization for layer aggregation}
Based on ULP, we augment the single confidence score $s_{L+1}$ computed from SID's final logits to a set of $L+1$ layerwise confidence scores $\{s_1, ..., s_L, s_{L+1}\}$. However, to perform SC, we need a single scalar as the output of the CSF $g$. Therefore, in this section, we discuss how to aggregate the $L+1$ confidence scores into a final confidence estimate.

For simplicity, we construct the final confidence estimate as a linear combination of all $L+1$ scores, i.e, $g_{\bm{w}}(x)=\sum_{l=1}^{L+1} w_l s_l$, where $\bm{w}=[w_1, ..., w_{L+1}]\in \mathbb{R}^{L+1}$. Assuming access to a hold-out validation set $V=\{(u_j, v_j)\}_{j=1}^M$, with inputs $u_j \in \mathcal{X}$ and labels $v_j \in \mathcal Y$, we optimize $\bm{w}$ on it by minimizing a preference optimization objective inspired by Eq.~(\ref{eq:po_obj})
\begin{small}
\begin{equation}
    \min_{\bm{w}} \Big \{ L_{ReSIDe}(g_{\bm{w}}, V) := \frac{1}{(M-n)n} \sum_{u^+ \in U^+} \sum_{u^- \in U^-} \ln(1 + \exp(g_{\bm{w}}(u^-) - g_{\bm{w}}(u^+))) \Big \}, \label{eq:reside}
\end{equation}\end{small}where $U^+=\{u \in V \mid f(u) = v\}$ and $U^-=\{u\in V \mid f(u) \neq v\}$ are the sets of correctly and wrongly classified samples, respectively, and $n=|U^-|$ is the number of wrongly classified samples.

Denote a solution to Eq.~(\ref{eq:reside}) as $\bm{w}^*$. As mentioned in Sec.~\ref{sec:bg_rw}, $g_{\bm{w}^*}$ tends to assign higher scores to correctly classified samples and lower scores to wrongly classified ones. In an extreme case, if $g_{\bm{w}^*}(u^+)>g_{\bm{w}^*}(u^-)$ for any $u^+\in U^+$ and $u^-\in U^-$, then we essentially achieve the optimal SC performance on $V$. To further justify the use of $L_{ReSIDe}$ in the context of SC, we show that $L_{ReSIDe}$ is in fact an upper bound of AURC in Thm.~\ref{thm:bound}, which is proved in Appendix \ref{app:proof_thm1}.

\begin{theorem}[Upper Bound of AURC via the ReSIDe Loss]
Let $L_{ReSIDe}$ be the ReSIDe loss defined in Eq.~(\ref{eq:reside}). Then, we have
\begin{small}
\begin{equation}
    \text{AURC}(g_{\bm{w}}, V) < \frac{2(M-n)n}{M^2\ln 2} L_{ReSIDe}(g_{\bm{w}}, V) + \frac{2n^2}{M^2}.
\end{equation}\end{small}Further assuming that each sample in $U^-$ has a unique confidence score, a tighter bound follows
\begin{small}\begin{equation}
    \text{AURC}(g_{\bm{w}}, V) < \frac{2(M-n)n}{M^2\ln 2} L_{ReSIDe}(g_{\bm{w}}, V) + \frac{n(n+1)}{M^2}.
\end{equation}\end{small}
\label{thm:bound}
\vspace{-5mm}
\end{theorem}

Therefore, minimizing the ReSIDe loss effectively minimizes an upper bound of the AURC, thus resulting in $g_{\bm{w}}$ with superior SC performance. To reduce complexity in the actual optimization process, we sample mini-batches of correctly and wrongly classified samples from $V$, denoted as $\mathcal{U}^+$ and $\mathcal{U}^-$, respectively, and minimize $L_{ReSIDe}(g_{\bm{w}}, \mathcal{U}^+ \cup\mathcal{U}^-)$ via gradient descent for each mini-batch until convergence or up to a fixed number of epochs. To conclude this section, we refer readers to Fig.~\ref{fig:pipeline} for an illustration of the overall ReSIDe framework.

\section{Experiments} \label{sec:experiments}
\subsection{Experimental settings}
To perform preference optimization, we curate seven hold-out validation sets based on GenImage, each corresponding to one of the testing distributions within ID, Blur, Unseen Cls., Low Res., FGSM, JPEG, and Unseen Gen. Note that in real-world scenarios, it's feasible for practitioners to construct hold-out validation sets corresponding to common types of covariate shift, but it's unrealistic to anticipate the exact covariate shifts in the wild. For example, at the time of preparing the validation sets, new generators, from which potential testing images are generated, may not have been invented yet. Therefore, to reflect this inevitable bias, there’s a distributional difference between the covariate shifts in our hold-out validation sets and those in the test sets. For instance, the validation set with Unseen Gen. contains synthetic images generated by GLIDE and VQDM, instead of those generated by the four generators used in the Unseen Gen. test set. More information on validation sets, training recipes, and full implementation details are provided in Appendix \ref{app:imp}.

As an approach aiming to fix the pathological behavior of logit-based CSFs, we benchmark ReSIDe against each underlying logit-based CSF applied only at the final layer. Given a logit-based CSF, we optimize $\bm{w}$ on a hold-out validation set with some covariate shift, and compare the AURC of $g_{\bm{w}^*}$ to that of the underlying final logit-based CSF on the test set with the same covariate shift. For comprehensiveness, this comparison is done for both SID models, all six types of logit-based CSFs and all six types of covariate shifts. Moreover, we also optimize $\bm{w}$ on the combination of all validation sets, and evaluate the resulting $g_{\bm{w}^*}$ under the mixture distribution mentioned in Sec.~\ref{sec:path_sc}.


\subsection{Main results}
Tab.~\ref{tab:aurc_six} shows superior performance of ReSIDe for each logit-based CSF, covariate shift, and SID model. Compared to the baseline of the underlying final logit-based CSF, ReSIDe achieves consistent and significant improvement in SC performance, reducing AURC by up to 69.55\%. Tab~\ref{tab:aurc_combined} compares ReSIDe to the baseline on the mixture distribution for all logit-based CSFs and both models. ReSIDe again outperforms the baseline in all cases, achieving up to 44.88\% AURC reduction.

\begin{table}[!t]
\centering
\caption{Comparison of AURCs under different covariate shifts. Results for ReSIDe are averaged over 3 trials with stds reported. $\Delta$ shows the AURC reduction achieved by ReSIDe in percentage.}
\label{tab:aurc_six}
\resizebox{0.95\textwidth}{!}{
\begin{tabular}{ccccccccccc}
\toprule
\multirow{2}{*}{Classifier} & \multirow{2}{*}{CSF} & \multicolumn{3}{c}{Blur} & \multicolumn{3}{c}{Unseen Cls.} & \multicolumn{3}{c}{Low Res.} \\
\cmidrule(lr){3-5} \cmidrule(lr){6-8} \cmidrule(lr){9-11}
 &  & Baseline & ReSIDe & $\Delta$ & Baseline & ReSIDe & $\Delta$ & Baseline & ReSIDe & $\Delta$ \\
\midrule

\multirow{6}{*}{ResNet50}
& MSP & 0.7276 & 0.2406\ms{0.0006} & \textcolor{ForestGreen}{66.93\%} & 0.0147 & 0.0095\ms{0.0000} & \textcolor{ForestGreen}{35.37\%} & 0.3817 & 0.3610\ms{0.0016} & \textcolor{ForestGreen}{5.42\%} \\
& SM  & 0.7276 & 0.2467\ms{0.0011} & \textcolor{ForestGreen}{66.09\%} & 0.0147 & 0.0095\ms{0.0001} & \textcolor{ForestGreen}{35.37\%} & 0.3817 & 0.3811\ms{0.0001} & \textcolor{ForestGreen}{0.16\%} \\
& ML  & 0.7271 & 0.2214\ms{0.0007} & \textcolor{ForestGreen}{69.55\%} & 0.0146 & 0.0128\ms{0.0001} & \textcolor{ForestGreen}{12.33\%} & 0.3782 & 0.3368\ms{0.0005} & \textcolor{ForestGreen}{10.95\%} \\
& LM  & 0.7276 & 0.2414\ms{0.0003} & \textcolor{ForestGreen}{66.82\%} & 0.0147 & 0.0092\ms{0.0001} & \textcolor{ForestGreen}{37.41\%} & 0.3817 & 0.3691\ms{0.0064} & \textcolor{ForestGreen}{3.30\%} \\
& NE  & 0.7276 & 0.2472\ms{0.0001} & \textcolor{ForestGreen}{66.03\%} & 0.0147 & 0.0096\ms{0.0003} & \textcolor{ForestGreen}{34.69\%} & 0.3817 & 0.3552\ms{0.0009} & \textcolor{ForestGreen}{6.94\%} \\
& NGI & 0.7276 & 0.2524\ms{0.0006} & \textcolor{ForestGreen}{65.31\%} & 0.0147 & 0.0096\ms{0.0002} & \textcolor{ForestGreen}{34.69\%} & 0.3817 & 0.3582\ms{0.0011} & \textcolor{ForestGreen}{6.16\%} \\

\midrule

\multirow{6}{*}{Swin-T}
& MSP & 0.6535 & 0.2889\ms{0.0003} & \textcolor{ForestGreen}{55.79\%} & 0.0047 & 0.0025\ms{0.0000} & \textcolor{ForestGreen}{46.81\%} & 0.4167 & 0.3255\ms{0.0004} & \textcolor{ForestGreen}{21.89\%} \\
& SM  & 0.6535 & 0.2379\ms{0.0007} & \textcolor{ForestGreen}{63.60\%} & 0.0047 & 0.0025\ms{0.0000} & \textcolor{ForestGreen}{46.81\%} & 0.4167 & 0.3178\ms{0.0001} & \textcolor{ForestGreen}{23.73\%} \\
& ML  & 0.6007 & 0.2311\ms{0.0006} & \textcolor{ForestGreen}{61.53\%} & 0.0059 & 0.0040\ms{0.0000} & \textcolor{ForestGreen}{32.20\%} & 0.3967 & 0.3236\ms{0.0038} & \textcolor{ForestGreen}{18.43\%} \\
& LM  & 0.6535 & 0.2187\ms{0.0003} & \textcolor{ForestGreen}{66.53\%} & 0.0047 & 0.0025\ms{0.0000} & \textcolor{ForestGreen}{46.81\%} & 0.4167 & 0.3073\ms{0.0035} & \textcolor{ForestGreen}{26.25\%} \\
& NE  & 0.6535 & 0.3079\ms{0.0006} & \textcolor{ForestGreen}{52.88\%} & 0.0047 & 0.0025\ms{0.0000} & \textcolor{ForestGreen}{46.81\%} & 0.4167 & 0.3293\ms{0.0000} & \textcolor{ForestGreen}{20.97\%} \\
& NGI & 0.6535 & 0.3154\ms{0.0001} & \textcolor{ForestGreen}{51.74\%} & 0.0047 & 0.0025\ms{0.0000} & \textcolor{ForestGreen}{46.81\%} & 0.4167 & 0.3349\ms{0.0004} & \textcolor{ForestGreen}{19.63\%} \\

\bottomrule
\toprule
\multirow{2}{*}{Classifier} & \multirow{2}{*}{CSF} & \multicolumn{3}{c}{FGSM} & \multicolumn{3}{c}{JPEG} & \multicolumn{3}{c}{Unseen Gen.} \\
\cmidrule(lr){3-5} \cmidrule(lr){6-8} \cmidrule(lr){9-11}
 &  & Baseline & ReSIDe & $\Delta$ & Baseline & ReSIDe & $\Delta$ & Baseline & ReSIDe & $\Delta$ \\
\midrule

\multirow{6}{*}{ResNet50}
& MSP & 0.4795 & 0.3449\ms{0.0009} & \textcolor{ForestGreen}{28.07\%} & 0.1922 & 0.1922\ms{0.0000} & 0.00\% & 0.3684 & 0.3630\ms{0.0000} & \textcolor{ForestGreen}{1.47\%} \\
& SM  & 0.4795 & 0.4030\ms{0.0003} & \textcolor{ForestGreen}{15.95\%} & 0.1922 & 0.1922\ms{0.0000} & 0.00\% & 0.3684 & 0.3630\ms{0.0000} & \textcolor{ForestGreen}{1.47\%} \\
& ML  & 0.4847 & 0.4501\ms{0.0009} & \textcolor{ForestGreen}{7.14\%} & 0.1925 & 0.1864\ms{0.0001} & \textcolor{ForestGreen}{3.17\%} & 0.3650 & 0.3650\ms{0.0000} & 0.00\% \\
& LM  & 0.4795 & 0.4035\ms{0.0011} & \textcolor{ForestGreen}{15.85\%} & 0.1922 & 0.1855\ms{0.0001} & \textcolor{ForestGreen}{3.49\%} & 0.3684 & 0.3630\ms{0.0000} & \textcolor{ForestGreen}{1.47\%} \\
& NE  & 0.4795 & 0.3461\ms{0.0001} & \textcolor{ForestGreen}{27.82\%} & 0.1922 & 0.1922\ms{0.0000} & 0.00\% & 0.3684 & 0.3630\ms{0.0000} & \textcolor{ForestGreen}{1.47\%} \\
& NGI & 0.4795 & 0.3467\ms{0.0002} & \textcolor{ForestGreen}{27.70\%} & 0.1922 & 0.1922\ms{0.0000} & 0.00\% & 0.3684 & 0.3630\ms{0.0000} & \textcolor{ForestGreen}{1.47\%} \\

\midrule

\multirow{6}{*}{Swin-T}
& MSP & 0.3023 & 0.2889\ms{0.0000} & \textcolor{ForestGreen}{4.43\%} & 0.1857 & 0.1764\ms{0.0001} & \textcolor{ForestGreen}{5.01\%} & 0.3480 & 0.3193\ms{0.0000} & \textcolor{ForestGreen}{8.25\%} \\
& SM  & 0.3023 & 0.3014\ms{0.0002} & \textcolor{ForestGreen}{0.30\%} & 0.1857 & 0.1727\ms{0.0002} & \textcolor{ForestGreen}{7.00\%} & 0.3480 & 0.3193\ms{0.0000} & \textcolor{ForestGreen}{8.25\%} \\
& ML  & 0.3022 & 0.2460\ms{0.0021} & \textcolor{ForestGreen}{18.60\%} & 0.1899 & 0.1642\ms{0.0009} & \textcolor{ForestGreen}{13.53\%} & 0.3737 & 0.3737\ms{0.0000} & 0.00\% \\
& LM  & 0.3023 & 0.2956\ms{0.0024} & \textcolor{ForestGreen}{2.22\%} & 0.1857 & 0.1698\ms{0.0000} & \textcolor{ForestGreen}{8.56\%} & 0.3480 & 0.3193\ms{0.0000} & \textcolor{ForestGreen}{8.25\%} \\
& NE  & 0.3023 & 0.2889\ms{0.0000} & \textcolor{ForestGreen}{4.43\%} & 0.1857 & 0.1720\ms{0.0001} & \textcolor{ForestGreen}{7.38\%} & 0.3480 & 0.3193\ms{0.0000} & \textcolor{ForestGreen}{8.25\%} \\
& NGI & 0.3023 & 0.2889\ms{0.0000} & \textcolor{ForestGreen}{4.43\%} & 0.1857 & 0.1729\ms{0.0002} & \textcolor{ForestGreen}{6.89\%} & 0.3480 & 0.3193\ms{0.0000} & \textcolor{ForestGreen}{8.25\%} \\

\bottomrule
\end{tabular}}
\end{table}

\begin{table}[!t]
\vspace{-5mm}
\centering
\begin{minipage}{0.43\textwidth}
\caption{Comparison of AURCs under the mixture distribution.}
\label{tab:aurc_combined}
\centering\resizebox{\textwidth}{!}{
\begin{tabular}{ccccc}
\toprule
Classifier & CSF & Baseline & ReSIDe & $\Delta$ \\
\midrule
\multirow{6}{*}{ResNet50}
& MSP & 0.3563 & 0.2004\ms{0.0008} & \textcolor{ForestGreen}{43.76\%} \\
& SM  & 0.3563 & 0.2072\ms{0.0006} & \textcolor{ForestGreen}{41.85\%} \\
& ML  & 0.3635 & 0.2266\ms{0.0034} & \textcolor{ForestGreen}{37.66\%} \\
& LM  & 0.3563 & 0.2094\ms{0.0014} & \textcolor{ForestGreen}{41.23\%} \\
& NE  & 0.3563 & 0.1972\ms{0.0003} & \textcolor{ForestGreen}{44.65\%} \\
& NGI & 0.3563 & 0.2028\ms{0.0035} & \textcolor{ForestGreen}{43.08\%} \\
\midrule
\multirow{6}{*}{Swin-T}
& MSP & 0.3347 & 0.1947\ms{0.0006} & \textcolor{ForestGreen}{41.83\%} \\
& SM  & 0.3347 & 0.1902\ms{0.0022} & \textcolor{ForestGreen}{43.17\%} \\
& ML  & 0.2868 & 0.2048\ms{0.0001} & \textcolor{ForestGreen}{28.59\%} \\
& LM  & 0.3347 & 0.1858\ms{0.0002} & \textcolor{ForestGreen}{44.49\%} \\
& NE  & 0.3347 & 0.1879\ms{0.0005} & \textcolor{ForestGreen}{43.86\%} \\
& NGI & 0.3347 & 0.1845\ms{0.0003} & \textcolor{ForestGreen}{44.88\%} \\
\bottomrule
\end{tabular}}
\end{minipage}
\hfil
\begin{minipage}{0.45\textwidth}
\caption{Comparison of AURCs under the mixture distribution.}
\label{tab:layer_combined}
\centering\resizebox{\textwidth}{!}{
\begin{tabular}{ccccc}
\toprule
Classifier & CSF & Best Layer & ReSIDe & $\Delta$ \\
\midrule
\multirow{6}{*}{ResNet50}
& MSP & 0.2554 & 0.2004\ms{0.0008} & \textcolor{ForestGreen}{21.53\%} \\
& SM  & 0.2554 & 0.2072\ms{0.0006} & \textcolor{ForestGreen}{18.87\%} \\
& ML  & 0.2580 & 0.2266\ms{0.0034} & \textcolor{ForestGreen}{12.17\%} \\
& LM  & 0.2554 & 0.2094\ms{0.0014} & \textcolor{ForestGreen}{18.01\%} \\
& NE  & 0.2554 & 0.1972\ms{0.0003} & \textcolor{ForestGreen}{22.79\%} \\
& NGI & 0.2554 & 0.2028\ms{0.0035} & \textcolor{ForestGreen}{20.60\%} \\
\midrule
\multirow{6}{*}{Swin-T}
& MSP & 0.2266 & 0.1947\ms{0.0006} & \textcolor{ForestGreen}{14.08\%} \\
& SM  & 0.2266 & 0.1902\ms{0.0022} & \textcolor{ForestGreen}{16.06\%} \\
& ML  & 0.2152 & 0.2048\ms{0.0001} & \textcolor{ForestGreen}{4.83\%} \\
& LM  & 0.2266 & 0.1858\ms{0.0002} & \textcolor{ForestGreen}{18.01\%} \\
& NE  & 0.2299 & 0.1879\ms{0.0005} & \textcolor{ForestGreen}{18.27\%} \\
& NGI & 0.2266 & 0.1845\ms{0.0003} & \textcolor{ForestGreen}{18.58\%} \\
\bottomrule
\end{tabular}}
\end{minipage}
\vspace{-3mm}
\end{table}

\subsection{Extensions}
\textbf{Orthogonality to logit transformation.} Tunable logit transformation are a collection of methods to improve post-hoc SC performance of logit-based CSFs. For example, pNorm \citep{cattelan2023fix, cattelan2023selective} applies centralization followed by $l_p$ normalization to logits to boost the SC performance of MaxLogit (ML), where $p$ is optimized by grid search on a hold-out validation set. Here, we show that ReSIDe can be used on top of pNorm to get further gain. Particularly, we apply pNorm not only to SID's final logits but also to the layerwsie logits extracted by ULP, while keeping other procedures of ReSIDe unchanged; we dub this combined method as ReSIDe-pNorm. Using ML as the underlying logit-based CSF, Tab.~\ref{tab:maxlogit_pnorm} compares the performance of ReSIDe-pNorm to the standard pNorm, which only utilizes SID's final logits. It's shown that ReSIDe-pNorm outperforms the standard pNorm in general, achieving up to 44.91\% AURC reduction. Another tunable logit transformation used in the SC literature is temperature scaling (TS) \citep{guo2017calibration, cattelan2023fix}. While TS can improve the SC performance of some logit-based CSFs in multi-class classification, it loses effectiveness for all of the six discussed logit-based CSFs in the case of binary classification, since the ranking of confidence scores is not influenced by TS. Therefore, TS is excluded from our discussion.

\begin{table}[!t]
\centering
\caption{Comparison of AURCs using ML as the underlying logit-based CSF.}
\label{tab:maxlogit_pnorm}
\resizebox{0.7\textwidth}{!}{
\begin{tabular}{cccccccc}
\toprule
& & Blur & Unseen Cls. & Low Res. & FGSM & JPEG & Unseen Gen. \\
\midrule
\multirow{3}{*}{ResNet50} & pNorm & 0.5144 & 0.0247 & 0.4834 & 0.1747 & 0.1453 & 0.1540 \\
& ReSIDe-pNorm & 0.2834 & 0.0204 & 0.4498 & 0.1747 & 0.1453 & 0.1540 \\
& $\Delta$ & \textcolor{ForestGreen}{44.91\%} & \textcolor{ForestGreen}{17.41\%} & \textcolor{ForestGreen}{6.95\%} & 0.00\% & 0.00\% & 0.00\% \\
\midrule
\multirow{3}{*}{Swin-T} & pNorm & 0.5370 & 0.0167 & 0.4928 & 0.2452 & 0.1503 & 0.1496 \\
& ReSIDe-pNorm & 0.3135 & 0.0144 & 0.3509 & 0.2182 & 0.1441 & 0.1471 \\
& $\Delta$ & \textcolor{ForestGreen}{41.62\%} & \textcolor{ForestGreen}{13.77\%} & \textcolor{ForestGreen}{28.79\%} & \textcolor{ForestGreen}{11.01\%} & \textcolor{ForestGreen}{4.13\%} & \textcolor{ForestGreen}{1.67\%} \\
\bottomrule
\end{tabular}}
\vspace{-2mm}
\end{table}

\textbf{Ablation study: the single best layer vs. the optimized linear combination.} The unified ReSIDe framework consists of two main components: (1) layerwsie score extraction based on ULP, and (2) score aggregation based on preference optimization. We hereby conduct an ablation study to analyze the contribution of each components. Given a logit-based CSF, we define the best layer as the layer with index $l^*=\argmin_{l\in [L+1]}\text{AURC}(g_{\bm{e}_l},V)$, where $\bm{e}_l \in \mathbb{R}^{L+1}$ denotes a one-hot vector whose \textit{l}-th entry is one. Then, we evaluate $\text{AURC}(g_{\bm{e}_{l^*}}, T)$ on the corresponding test set. Tab.~\ref{tab:layer_combined} shows the performance of the best layer for each logit-based CSF and SID model under the mixture distribution. Comparing the best layer's results to those of the baseline presented in Tab.~\ref{tab:aurc_combined}, we see that the best layer consistently outperforms the baseline. Nonetheless, ReSIDe provides further improvement over the best layer, achieving up to 22.79\% AURC reduction. This shows that both components within ReSIDe contribute meaningfully to its superior performance.

\begin{figure}[!t]
\centering
\begin{subfigure}[b]{0.22\textwidth}  
\centering 
\includegraphics[width=\textwidth]{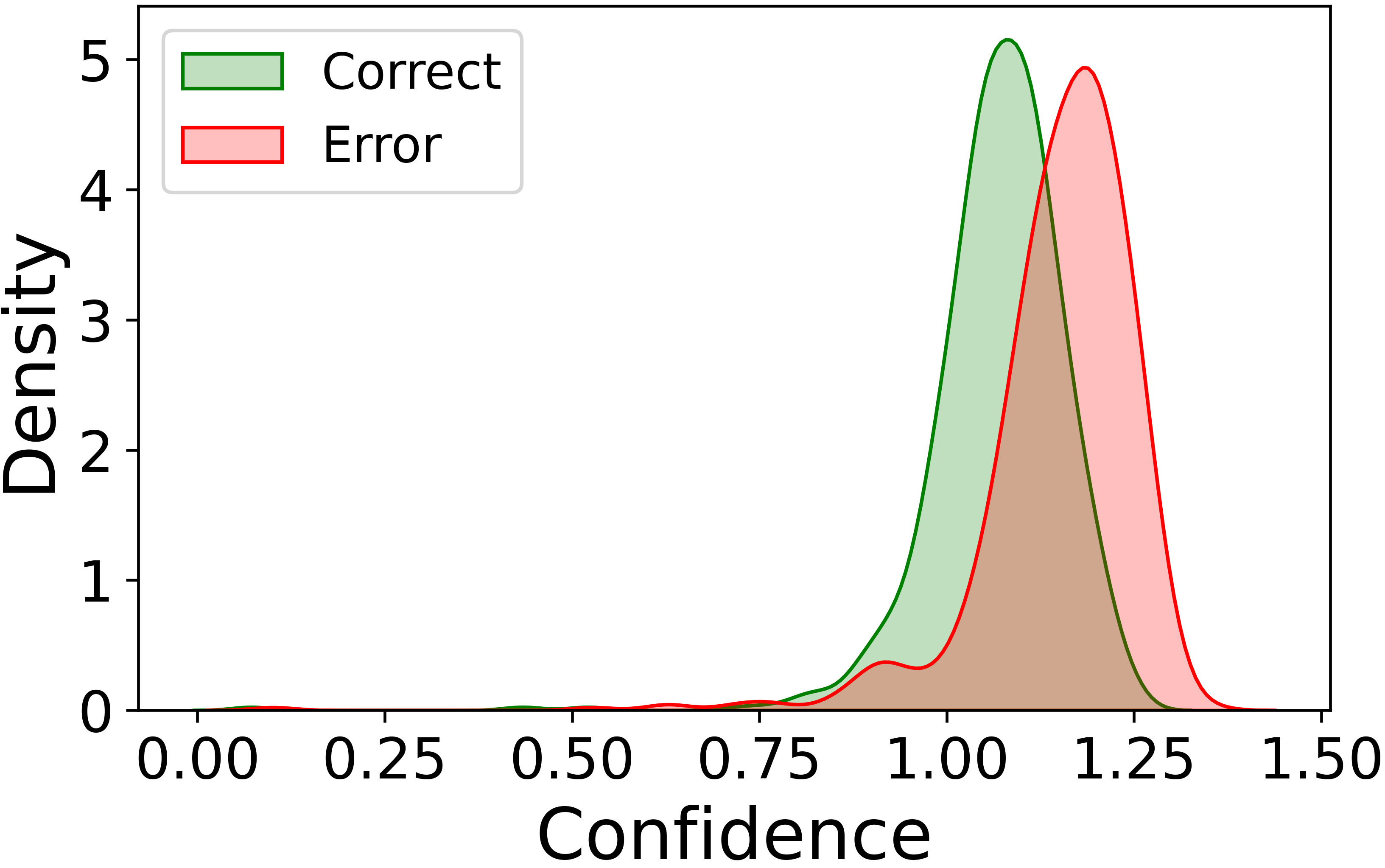}
\caption*{(a) Baseline, ResNet50}
\end{subfigure}
\begin{subfigure}[b]{0.22\textwidth}  
\centering 
\includegraphics[width=\textwidth]{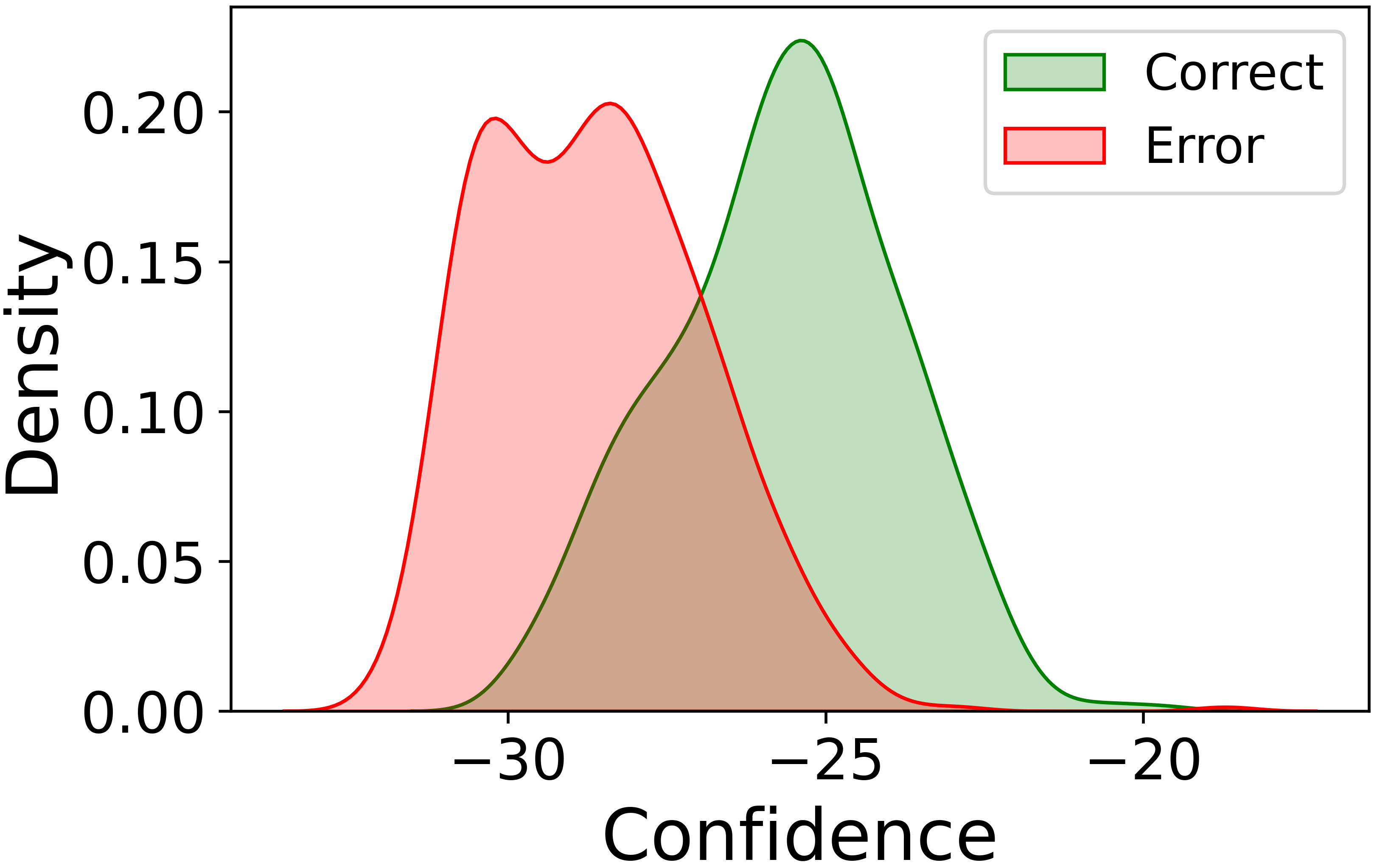}
\caption*{(b) ReSIDe, ResNet50}
\end{subfigure}
\begin{subfigure}[b]{0.22\textwidth}  
\centering 
\includegraphics[width=\textwidth]{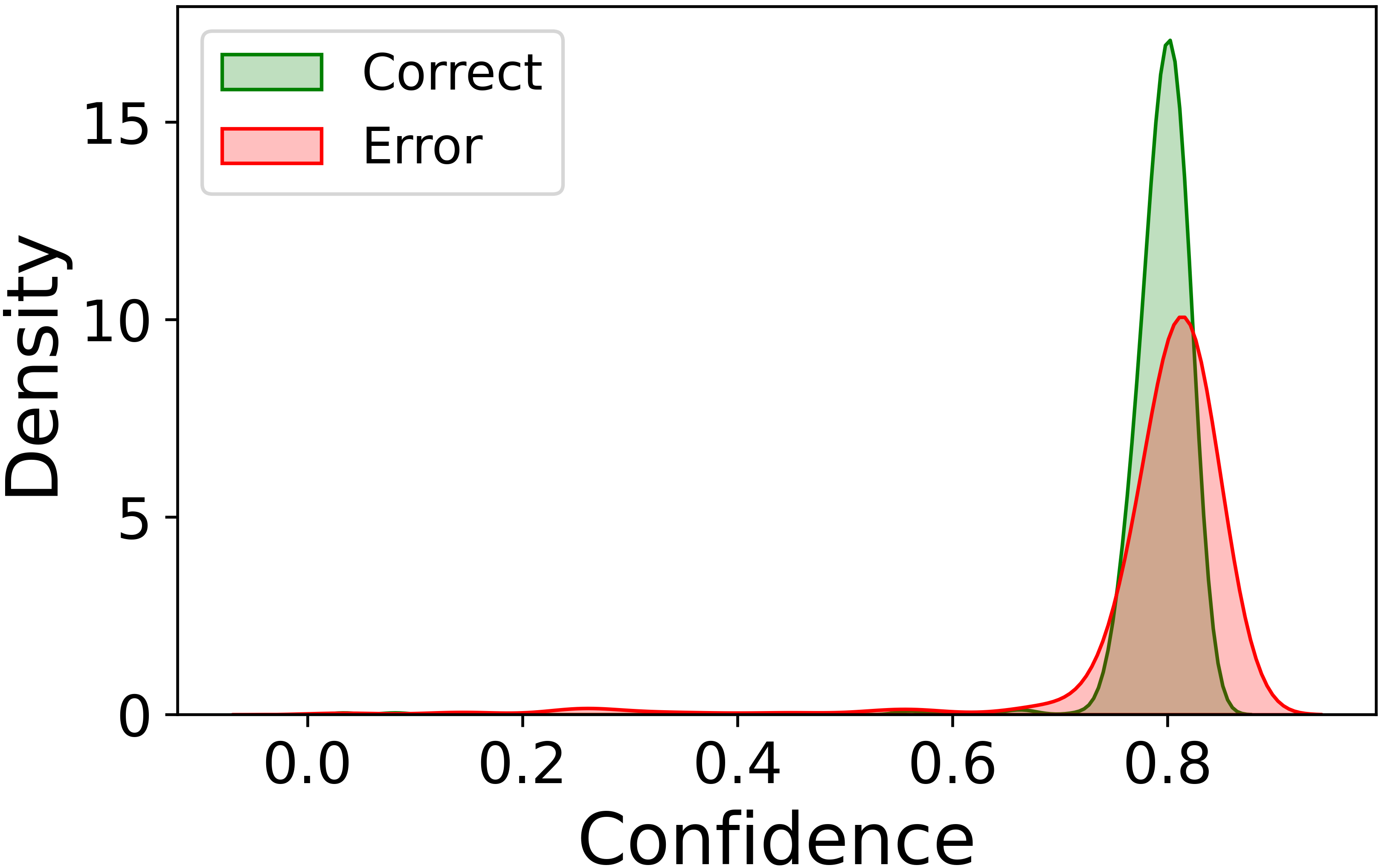}
\caption*{(c) Baseline, Swin-T}
\end{subfigure}
\begin{subfigure}[b]{0.22\textwidth}  
\centering 
\includegraphics[width=\textwidth]{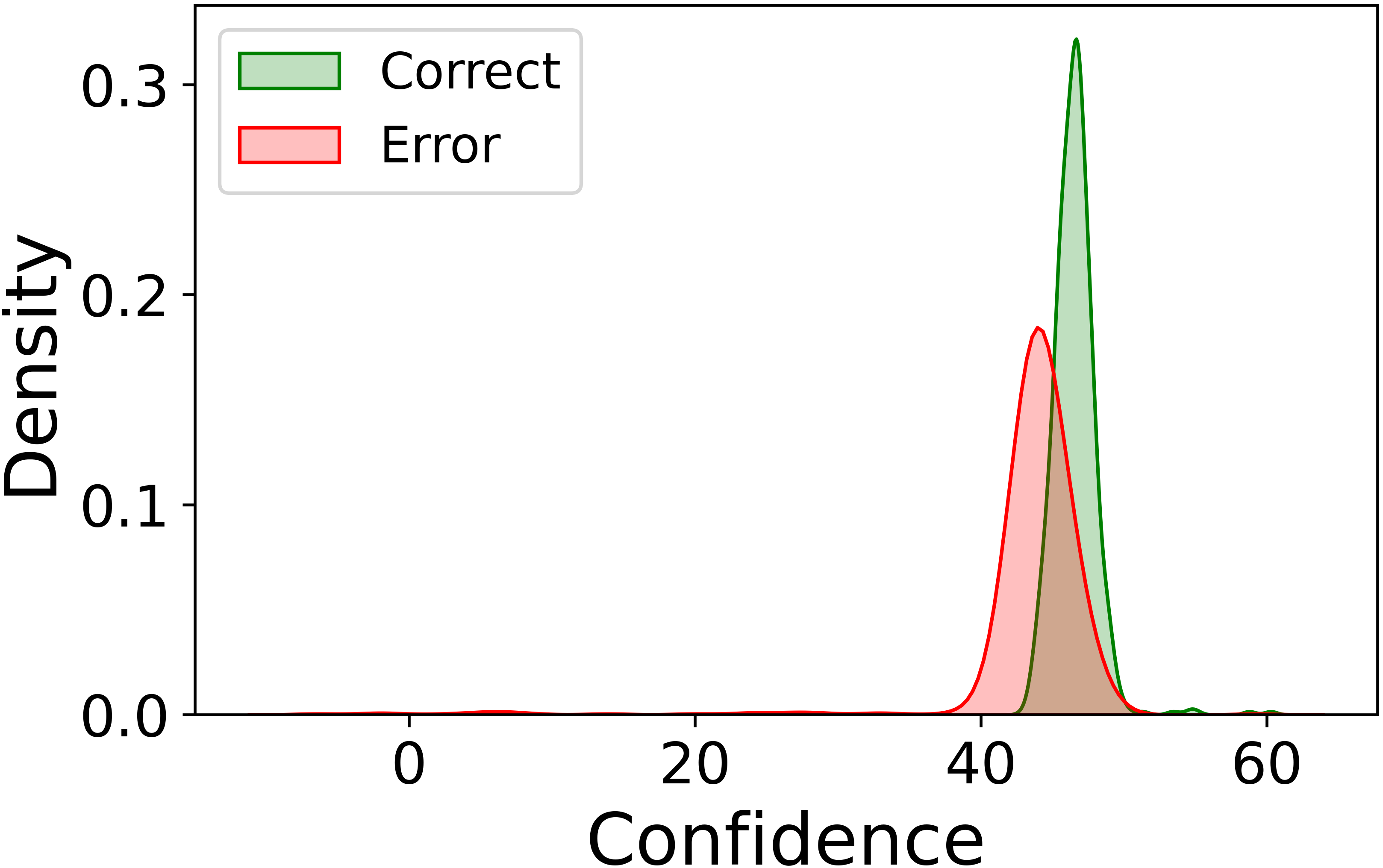}
\caption*{(d) ReSIDe, Swin-T}
\end{subfigure}
\vspace{-1mm}
\caption{Comparison of confidence score distributions between the baseline and ReSIDe.} 
\label{fig:score_dist}
\vspace{-5mm}
\end{figure}

\textbf{Distribution of confidence scores.} Since the ranking of confidence scores determines the SC performance, it's informative to analyze the confidence score distributions of correctly classified samples and wrongly classified samples. Figs.~\ref{fig:score_dist}(a) and (b) compare the score distributions of ReSIDe and the baseline using ML as the underlying logit-based CSF, evaluated for ResNet50 in the case of Blur. As shown, wrongly classified samples generally get higher scores than correctly classified ones in the baseline, causing the pathological SC behavior identified in Fig.~\ref{fig:rc}; while ReSIDe addresses the issue by assigning generally higher scores to correctly classified samples. A similar analysis is done for Swin-T in Figs.~\ref{fig:score_dist}(c) and (d), with SM as the underlying logit-based CSF, also evaluated in the case of Blur. Again, ReSIDe successfully flips the ordering between correctly classified samples and wrongly classified ones, on average giving the former higher scores than the latter. This analysis explains why ReSIDe substantially improves the SC performance of the baseline.

More analysis on the \textbf{empirical convergence behavior of our preference optimization algorithm} and the \textbf{computational complexity} are shown in Appendix \ref{app:converge} and \ref{app:complex}, respectively.

\section{Conclusion} \label{sec:con_lim}
The paper identifies a pathological SC behavior of logit-based CSFs when applied to SIDs under common covariate shifts. To fix this issue, ReSIDe is proposed to generalize the use of logit-based CSFs to SID's intermediate layers. Using a linear combination of layerwise confidence scores as the final confidence estimate, ReSIDe can significantly improve the SC performance of logit-based CSFs. Comprehensive experimental results show that ReSIDe is effective over all kinds of covariate shifts, logit-based CSFs and SID models.

\textbf{Limitations:} (1) Due to the use of intermediate layerwise logits, ReSIDe is not applicable to non-DNN-based SIDs; however, the aggregation component of our method is applicable whenever multiple confidence scores can be extracted from the SID. (2) The performance of the preference optimization algorithm depends on the hold-out validation set. Thus, practitioners need to make sure the validation distribution is sufficiently varied as to not overfit to any particular covariate shift. (3) Lastly, the creation of the feature centroids requires access to samples from the training distribution. When this is not available, one would need to sample instead from an approximation to this distribution, which could introduce additional approximation errors into the method.



\bibliographystyle{plainnat}
\bibliography{references}


\appendix

\section{Related work} \label{app:rw}
\textbf{Synthetic image detection.}
With the rapid progress of modern generative models such as GANs and diffusion models, synthetic image detection has become increasingly important to prevent the misuse of realistic synthetic images. The goal of this task is to build a synthetic image detector (SID) that distinguishes synthetic images from real ones. I.e., given an input image, an SID is a binary classifier that predicts either ``real'' or ``fake''. Although some works show promising results for SIDs based on hand-crafted features \citep{frank2020leveraging, cozzolino2024zero, ricker2024aeroblade}, modern SIDs are predominately based on DNNs trained by supervised learning \citep{rossler2019faceforensics++, wang2020cnn, yan2023ucf, sha2023fake, koutlis2024leveraging, maier2024reliable}, with many demonstrating near-perfect in-distribution (ID) detection performance. However, recent studies show that these DNN-based SIDs often fail to generalize beyond their training distribution \citep{yumlembam2025detection, maier2024reliable, nguyen2025forensic}, making their reliability questionable in actual deployment. In real-world scenarios, inputs may undergo common covariate shifts such as blurring, JPEG compression, low resolution, adversarial perturbations, images generated by unseen generative models, or images containing unseen semantics. Under these conditions, the performance of SIDs can degrade significantly. To tackle this, recent work proposes to train SIDs with a reject option \citep{maier2024reliable}; however, this approach requires retraining the SID, arguably less practical and efficient than post-hoc methods that don't modify the weights of existing DNN-based SIDs. \cite{yumlembam2025detection} address the issue by out-of-distribution (OOD) detection, which is a weak form of selective classification (SC) under distribution shifts\footnote{OOD detection aims to enhance the reliability of a classifier by rejecting OOD samples. However, OOD samples may be correctly classified, while ID samples may be wrongly classified. Thus, if our goal is to achieve reliable predictions by rejecting samples which are likely to be wrongly classified, we should directly adopt the SC framework, instead of relying on OOD detection with a misaligned objective.} \citep{liang2024selective}. Moreover, most existing works only focus on limited types of covariate shifts such as unseen generators and JPEG compression \citep{maier2024reliable, yumlembam2025detection, nguyen2025forensic}. In contrast, this paper proposes a post-hoc SC framework to address SID's reliability issue, considering six types of covariate shifts, thus being practical, efficient, general, and comprehensive.

\textbf{Selective classification.} Selective classification augments a standard classifier with a rejection option, allowing the model to abstain from making predictions on inputs it is uncertain about. Instead of forcing a prediction for every input, the model produces both a prediction and a confidence score given an input, and outputs the prediction only if the confidence exceeds a predefined threshold. This cautious and conservative prediction framework is often desirable in safety-critical scenarios such as medical diagnostics, autonomous driving, and the justice system \citep{zou2023review, neumann2018relaxed, berk2021fairness}, where any misclassification could lead to severe consequences. Existing SC approaches can be broadly categorized into \textit{training-based} methods and \textit{post-hoc} methods. Training-based approaches \citep{geifman2019selectivenet, huang2020self, liu2019deep, zhou2024novel} consider the classifier as part of the learning problem to optimize for the SC performance, while post-hoc methods \citep{franc2023optimal, liang2024selective} assume that the classifier is pretrained and fixed, thus focusing exclusively on confidence estimation based on a fixed model. Since retraining the classifier could be costly or infeasible in practice, post-hoc SC methods are particularly attractive in real-world deployment scenarios, thus being the primary focus of this paper.

\textbf{Confidence estimation}. Towards better confidence estimation in SC and other applications, various CSFs have been proposed in the literature. The first line of work focuses on ensemble methods \citep{lakshminarayanan2017simple, teye2018bayesian, liu2023spectral}, with Monte-Carlo dropout \citep{gal2016dropout} being a prominent example. However, ensemble methods necessitate multiple forward passes to provide a confidence estimate, which is considered inefficient in practice. Also, recent work suggests that ensemble methods may not contribute much to confidence estimation, but to building a better performing predictor \citep{abe2022deep, xia2022usefulness}. Some other works propose to train auxiliary DNN models to serve as CSFs \citep{shen2023post, corbiere2021confidence}. However, such DNNs can be as large as the base classifier in the SC system, making their training inefficient. Therefore, we shift our focus to another line of work on logit-based CSFs, which are arguably the most efficient class of CSFs. Logit-based CSFs compute confidence scores using classifier's output logits. For example, since DNNs routinely adopt softmax over logits to provide a probabilistic output, the maximum softmax probability (MSP) \citep{hendrycks2016baseline}, also known as the softmax response \citep{geifman2017selective}, is naturally used as a logit-based CSF. Other logit-based CSFs include softmax margin (SM) \citep{belghazi2021classifiers}, the max logit (ML) \citep{basart2022scaling}, the logits margin (LM) \citep{le1990handwritten, liang2024selective}, the negative entropy (NE) \citep{belghazi2021classifiers}, and the negative Gini index (NGI) \citep{granese2021doctor}. In this paper, we evaluate the SC performance of these logit-based CSFs in the context of synthetic image detection under distribution shifts, and propose an approach to improve their performance.

\section{Reweighting for dataset mixture} \label{app:reweight}
Let the test set $T$ be the union of $H$ disjoint subsets, i.e., $T=T_1 \cup T_2 \cup \dots \cup T_H$, where $|T| = N$ and $|T_m| = N_m$ for $m \in [H]$. We reweight samples such that each subset takes up a total mass of $1/H$, ensuring an even mixture. Concretely, the empirical coverage, the empirical selective risk, and AURC are computed by
$$\hat{\phi}_t(f, g, T) = \frac{1}{H} \sum_{m=1}^{H} \sum_{i=1}^{N_m} \mathbbm{1}[g(x_{m,i}) \geq t] / N_m,$$
$$\hat{R}_t(f, g, T) = \frac{\frac{1}{H} \sum_{m=1}^{H} \sum_{i=1}^{N_m} l(f(x_{m,i}), y_{m,i}) \mathbbm{1}[g(x_{m,i}) \geq t] / N_m}{\hat{\phi}_t(f, g, T)},$$
and
$$\text{AURC}(f, g, T) = \frac{1}{H} \sum_{m=1}^{H} \sum_{i=1}^{N_m} \hat{R}_{g(x_{m,i})}(f, g, T)/ N_m,$$
respectively, where $(x_{m,i},y_{m,i})$ denotes the \textit{i}-th sample from the \textit{m}-th subset. It can be seen that if a sample belongs to the \textit{m}-th subset, its mass is $1/HN_m$.

In practice, our implementation of RC curve plotting and AURC computation follows some packages, e.g., \textit{torch-uncertainty}\footnote{https://github.com/ENSTA-U2IS-AI/
torch-uncertainty}, where we sort samples based on their confidence scores and perform (1) a cumulative sum over sample masses to get all coverage levels, (2) a cumulative sum of weighted errors, i.e., 0/1 losses weighted by corresponding sample masses, divided by corresponding coverages, to get error rates at all coverage levels, and (3) a weighted sum or error rates to get AURC. A PyTorch-style pseudo-code is provided in Algorithm \ref{alg:reweight}.

\begin{algorithm}[h!]
\caption{PyTorch-style pseudo-code for reweighted AURC computation.}
\label{alg:reweight}
\footnotesize
\begin{alltt}
\color{ForestGreen}
# confidence: a list of N confidence scores for all samples
# error_flag: a list of N 0/1 losses for all samples
# mass: a list of N masses for all samples
\end{alltt}
\begin{alltt}
\textcolor{ForestGreen}{# Sort samples by confidence (descending)}
sorted_conf, indices = torch.sort(confidence, descending=True)
sorted_error = error_flag[indices]
sorted_mass = mass[indices]

\textcolor{ForestGreen}{# Cumulative statistics}
cum_error = torch.cumsum(sorted_error, dim=0)
coverage = torch.cumsum(sorted_mass, dim=0)

\textcolor{ForestGreen}{# Selective risks at all coverages}
risk = cum_error / coverage

\textcolor{ForestGreen}{# Reweighted AURC}
AURC = (risk * sorted_mass).sum()
\end{alltt}
\end{algorithm}

\section{Proof of Theorem \ref{thm:bound}} \label{app:proof_thm1}

\begin{proof}
Define the ranking error as $L_{rank}(g_{\bm{w}}, V) = \frac{1}{(M-n)n} \sum_{u^+ \in U^+} \sum_{u^- \in U^-} \mathbbm{1}[g_{\bm{w}}(u^-) \ge g_{\bm{w}}(u^+)]$. Since $\mathbbm{1}[q \ge 0] \le \frac{\ln(1+e^q)}{\ln 2}$ for any $q \in \mathbb{R}$, we have
\begin{equation}
    L_{rank}(g_{\bm{w}}, V) \le \frac{1}{\ln 2} L_{ReSIDe}(g_{\bm{w}}, V). \label{eq:soft_hard}
\end{equation}

Then, we invoke the SELE loss $\Delta_{sele}(g_{\bm{w}}, V) = \frac{1}{M^2} \sum_{u^- \in U^-} \sum_{u \in V} \mathbbm{1}[g_{\bm{w}}(u^-) \ge g_{\bm{w}}(u)]$ from \cite{franc2023optimal}, and decompose it as follows
\begin{align}
    \Delta_{sele}(g_{\bm{w}}, V) &= \frac{1}{M^2} \sum_{u^- \in U^-} \left[ \sum_{u^+ \in U^+} \mathbbm{1}(g_{\bm{w}}(u^-) \ge g_{\bm{w}}(u^+)) + \sum_{\hat{u}^- \in U^-} \mathbbm{1}(g_{\bm{w}}(u^-) \ge g_{\bm{w}}(\hat{u}^-)) \right] \nonumber \\
    &= \frac{(M-n)n}{M^2} L_{rank}(g_{\bm{w}}, V) + \frac{1}{M^2} \sum_{u^- \in U^-}\sum_{\hat{u}^- \in U^-} \mathbbm{1}(g_{\bm{w}}(u^-) \ge g_{\bm{w}}(\hat{u}^-)) \nonumber \\
    &\leq \frac{(M-n)n}{M^2} L_{rank}(g_{\bm{w}}, V) + \frac{n^2}{M^2}, \label{eq:wrong_internal}
\end{align}
where the equality in Eq.~(\ref{eq:wrong_internal}) holds if and only if all wrongly classified samples have the same confidence score, i.e., $g_{\bm{w}}(u^-) = g_{\bm{w}}(\hat{u}^-)$ for any $u^- \in U^-$ and $\hat{u}^- \in U^-$. 

Based on Eq.~(\ref{eq:soft_hard}) and Eq.~(\ref{eq:wrong_internal}), we have
\begin{equation}
    \Delta_{sele}(g_{\bm{w}}, V) \le \frac{(M-n)n}{M^2 \ln 2} L_{ReSIDe}(g_{\bm{w}}, V) + \frac{n^2}{M^2}.
\end{equation}

Further applying Theorem 8 from \cite{franc2023optimal}, which states $\text{AURC}(g_{\bm{w}}, V) < 2\Delta_{sele}(g_{\bm{w}}, V)$, we get
\begin{align}
    \text{AURC}(g_{\bm{w}}, V) < \frac{2(M-n)n}{M^2\ln 2} L_{ReSIDe}(g_{\bm{w}}, V) + \frac{2n^2}{M^2}.
\end{align}

If we assume that each sample in $U^-$ has a unique score, then we can tighten Eq.~(\ref{eq:wrong_internal}) to
\begin{equation}
    \Delta_{sele}(g_{\bm{w}}, V) = \frac{(M-n)n}{M^2} L_{rank}(g_{\bm{w}}, V) + \frac{n(n+1)}{2M^2}, \label{eq:wrong_internal_unique}
\end{equation}
and therefore we have 
\begin{align}
    \text{AURC}(g_{\bm{w}}, V) < \frac{2(M-n)n}{M^2\ln 2} L_{ReSIDe}(g_{\bm{w}}, V) + \frac{n(n+1)}{M^2}.
\end{align}

This completes the proof of Theorem \ref{thm:bound}.
\end{proof}

\section{Implementation details} \label{app:imp}
\textbf{Hold-out validation sets.}
The ID validation set follows the same distribution as the training set, i.e., containing images from the same sources and semantic classes, with in total 1000 images. The remaining six validation sets have various covariate shifts, each containing 1000 images as well:
\begin{itemize}
    \item \textbf{Blur:} images are processed by Gaussian blurring with $\sigma=5$;
    \item \textbf{Unseen semantic classes (Unseen Cls.):} images from 50 semantic classes not presented in training;
    \item \textbf{Low resolution (Low Res.):} images are resized while preserving the aspect ratio, such that the length of the shorter edge is 64;
    \item \textbf{FGSM:} images are perturbed using FGSM, with $\epsilon=8/255$;
    \item \textbf{JPEG:} images are compressed with JPEG with $q=30$;
    \item \textbf{Unseen generators (Unseen Gen.):} synthetic images are generated by GLIDE and VQDM, which are not presented in training.
\end{itemize}

Note that there's a specific reason why we train SIDs using synthetic images from ADM and BigGAN, validate using GLIDE and VQDM, and test using Wukong, Midjourney, Stable Diffusion V1.4, and Stable Diffusion V1.5. In fact, to reflect the realistic scenario, we order image generators chronologically, and use the oldest generators for training SIDs, relatively newer generators for validation, and latest generators for testing. Also, it's worth noting that images in the hold-out validation sets are neither used in the training of SIDs nor included in the test sets.

\textbf{Details of ULP.} For ResNet50, we conduct ULP after each residual block, and therefore $L=16$. Note that we apply global average pooling (GAP) over feature maps to get the feature vector $\bm{h}_l$. Base on silhouette analysis, we get $[K_l]_{l=1}^L=[4,3,3,4,4,4,5,4,4,4,4,3,3,2,2,2]$. It can be seen that the number of centroids generally decreases as the layer gets deeper. In fact, for all three blocks in the last stage of ResNet50, the number of centroids stays at 2, the same as the number of classes in synthetic image detection. This behavior is consistent with our intuition, where image representations are more dispersed at early DNN layers, while become concentrated to 2 clusters at deeper layers for binary classification.

For Swin-T, we conduct ULP after each swin transformer block, and therefore $L=12$. Similarly, GAP is adopted over feature maps to get the feature vector. Base on silhouette analysis, we get $[K_l]_{l=1}^L=[3,2,2,2,3,3,3,2,2,2,2,2]$. A similar trend is exhibited as ResNet50, where earlier layers has generally more centroids than deeper ones. Again, for the 5 deepest blocks in Swin-T, the number of centroids stays at 2, the same as the number of classes in synthetic image detection.

\textbf{The training recipe of the preference optimization algorithm.} To optimize $\bm{w}$ following Eq.~\ref{eq:reside}, we adopt an Adam optimizer with a learning rate of 0.1. The batch size is 16, i.e., $|\mathcal{U}^+|=|\mathcal{U}^-|=16$, and we train for 100 epochs. The training wall-clock time is in the order of units of minutes using an NVIDIA V100 GPU.

Note that we initialize $\bm{w}$ with $\bm{e}_{l^*}$, where $l^*=\argmin_{l\in [L+1]}\text{AURC}(g_{\bm{e}_l},V)$ is the index of the best layer, and $\bm{e}_l \in \mathbb{R}^{L+1}$ denotes a one-hot vector whose \textit{l}-th entry is one. Since the preference optimization algorithm returns the checkpoint of $\bm{w}$ with the lowest validation AURC, this initialization enables an implicit fallback to the single best layer. If combining confidence scores from other layers do not provide AURC reduction compared to the single best layer, the final returned checkpoint will just be the initialization $\bm{e}_{l^*}$, and thus the performance of ReSIDe will be the same as the best layer. In a special case where $l^*=L+1$ and $\bm{e}_{l^*}$ is returned by the preference optimization algorithm, the performance of ReSIDe will be the same as the baseline. This baseline fallback behavior explains why some results of ReSIDe in Tabs.~\ref{tab:aurc_six} and \ref{tab:maxlogit_pnorm} get $0.00\%$ AURC reduction.

\section{Convergence curves} \label{app:converge}
Fig.~\ref{fig:convergence}(a) plots the convergence curves of the ReSIDe loss and AURC on the validation set using ML as the underlying logit-based CSF, evaluated for ResNet50 in the case of Blur. It's shown that as the preference optimization progresses, both the ReSIDe loss and AURC decrease. This indicates that: (1) our proposed preference optimization algorithm based on Eq.~(\ref{eq:reside}) effectively minimizes the ReSIDe loss, and (2) as an upper bound of AURC, the minimization of the ReSIDe loss effectively lead to the minimization of AURC. A similar analysis is done for Swin-T in Fig.~\ref{fig:convergence}(b), with SM as the underlying logit-based CSF, also evaluated in the case of Blur. Again, both the ReSIDe loss and AURC converge successfully. Therefore, this analysis verifies the effectiveness of our preference optimization algorithm.

\begin{figure}[!ht]
\centering
\begin{subfigure}[b]{0.4\textwidth}  
\centering 
\includegraphics[width=\textwidth]{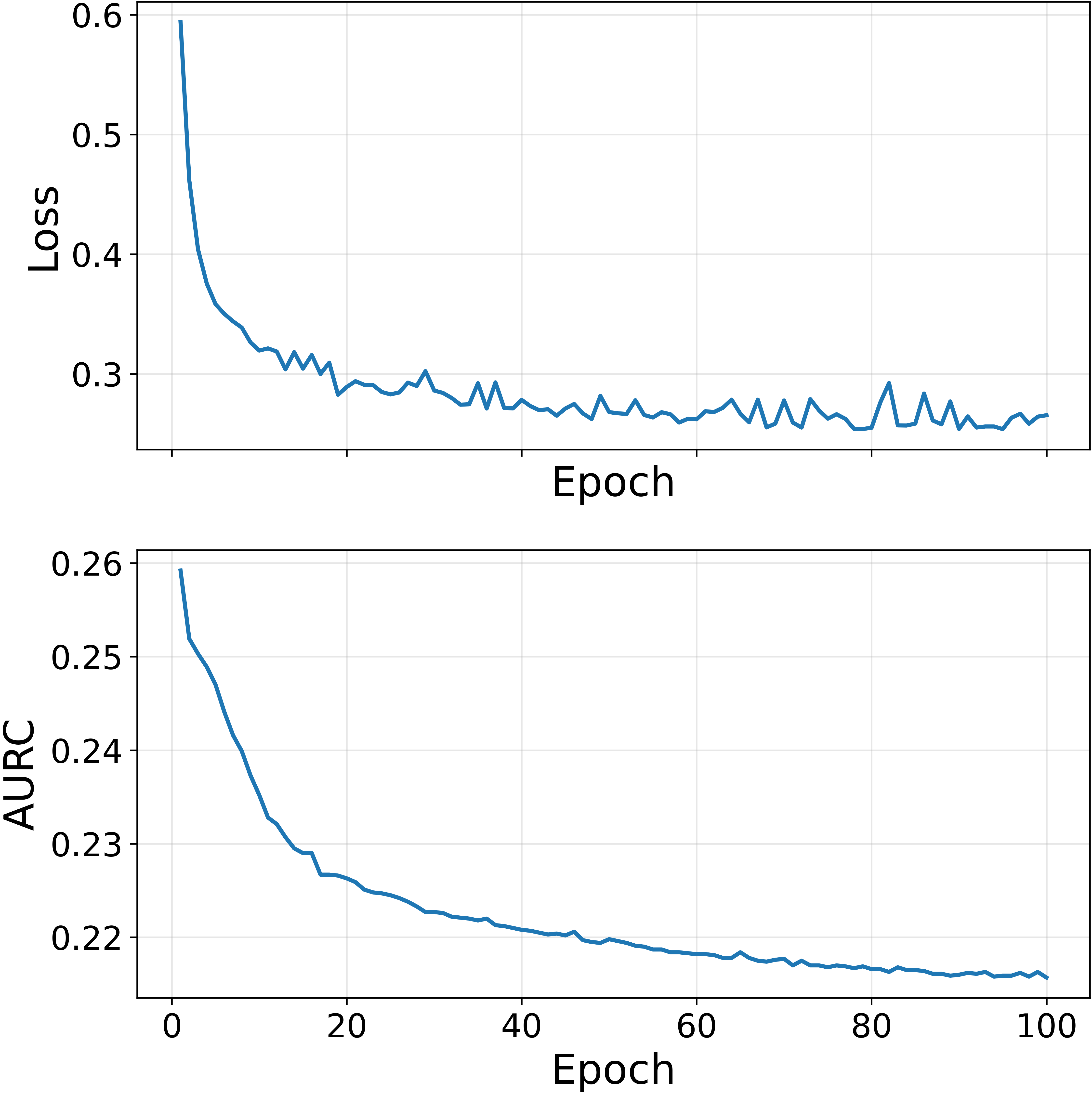}
\caption*{(a) ResNet50}
\end{subfigure}
\begin{subfigure}[b]{0.4\textwidth}
\centering 
\includegraphics[width=\textwidth]{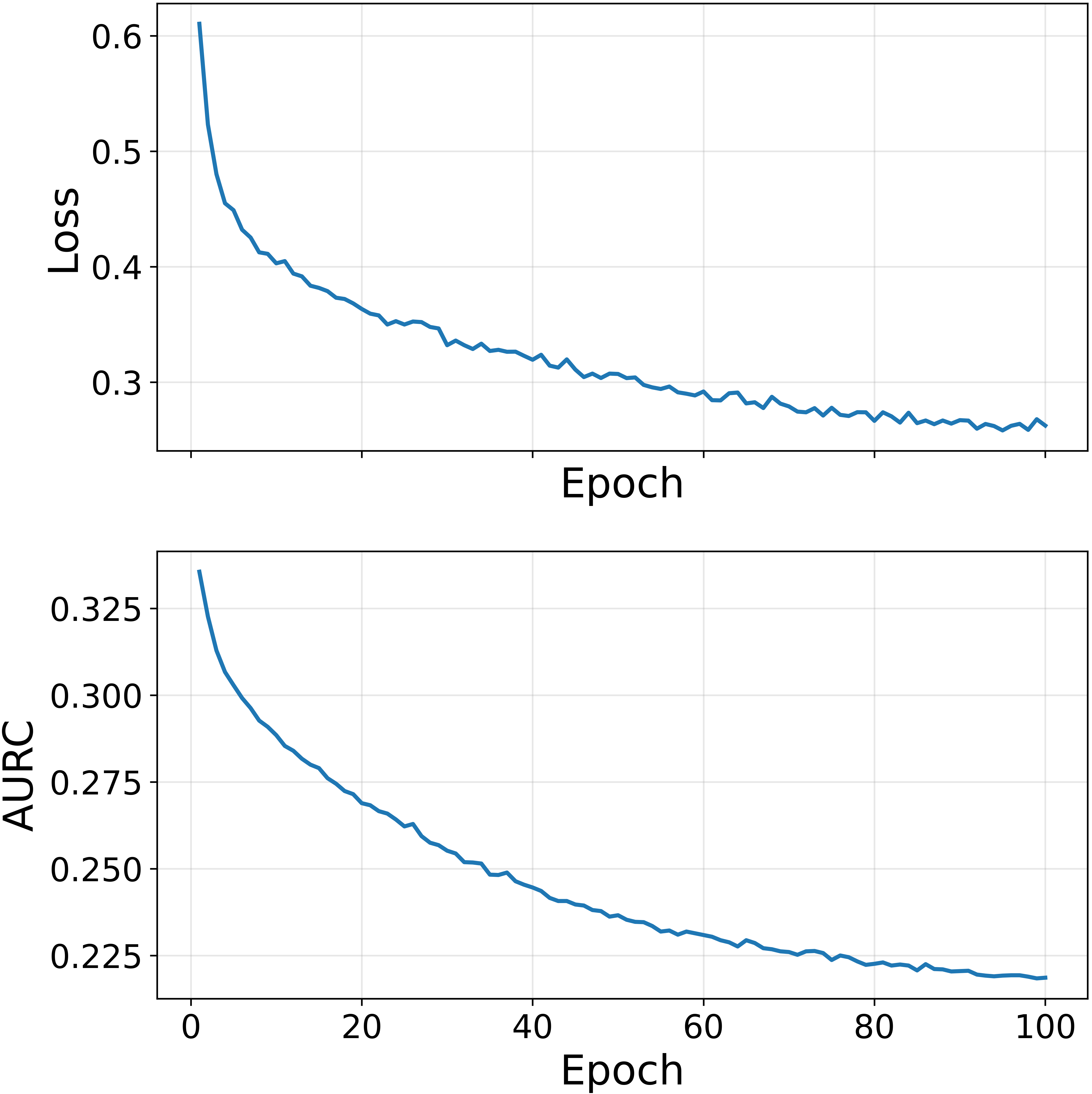}
\caption*{(b) Swin-T}
\end{subfigure}
\caption{Convergence curves of the ReSIDe loss and AURC.} 
\label{fig:convergence}
\end{figure}

\section{Computational complexity} \label{app:complex}
Tab.\ref{tab:comp} shows the inference latency of selective classifiers based on ReSIDe, measured on an NVIDIA L4 GPU. Compared to the baseline, i.e., selective classifiers based on the underlying final logit-based CSF, ReSIDe just slightly increases the latency by approximately 5\% on both models, largely maintaining the efficiency of the SC system.

\begin{table}[!ht]
\centering
\caption{Comparison of inference latency, measured in milliseconds per image (ms/img).}
\label{tab:comp}
\centering\resizebox{0.38\textwidth}{!}{
\begin{tabular}{cccc}
\toprule
Classifier & CSF & Baseline & ReSIDe \\
\midrule
\multirow{6}{*}{ResNet50}
& MSP & 1.996 & 2.091 \\
& SM  & 2.004 & 2.103 \\
& ML  & 1.994 & 2.066 \\
& LM  & 2.008 & 2.085 \\
& NE  & 2.002 & 2.096 \\
& NGI & 2.009 & 2.113 \\
\midrule
\multirow{6}{*}{Swin-T}
& MSP & 2.705 & 2.856 \\
& SM  & 2.730 & 2.907 \\
& ML  & 2.717 & 2.873 \\
& LM  & 2.719 & 2.894 \\
& NE  & 2.704 & 2.886 \\
& NGI & 2.735 & 2.908 \\
\bottomrule
\end{tabular}}
\end{table}


\newpage
\section*{NeurIPS Paper Checklist}

\begin{enumerate}

\item {\bf Claims}
    \item[] Question: Do the main claims made in the abstract and introduction accurately reflect the paper's contributions and scope?
    \item[] Answer: \answerYes{} 
    \item[] Justification: The main claims in the abstract and introduction are supported by the theoretical and experimental results presented in Sections \ref{sec:eval_csf}, \ref{sec:method} and \ref{sec:experiments}.

\item {\bf Limitations}
    \item[] Question: Does the paper discuss the limitations of the work performed by the authors?
    \item[] Answer: \answerYes{} 
    \item[] Justification: The limitations of the work are discussed in Section \ref{sec:con_lim}.

\item {\bf Theory assumptions and proofs}
    \item[] Question: For each theoretical result, does the paper provide the full set of assumptions and a complete (and correct) proof?
    \item[] Answer: \answerYes{} 
    \item[] Justification: The proof of Theorem 1 is presented in Appendix \ref{app:proof_thm1}, with assumptions clearly stated.

    \item {\bf Experimental result reproducibility}
    \item[] Question: Does the paper fully disclose all the information needed to reproduce the main experimental results of the paper to the extent that it affects the main claims and/or conclusions of the paper (regardless of whether the code and data are provided or not)?
    \item[] Answer: \answerYes{} 
    \item[] Justification: The experimental settings and full details are presented in Sections \ref{sec:eval_csf}, \ref{sec:experiments} and Appendix \ref{app:imp}.

\item {\bf Open access to data and code}
    \item[] Question: Does the paper provide open access to the data and code, with sufficient instructions to faithfully reproduce the main experimental results, as described in supplemental material?
    \item[] Answer: \answerNo{} 
    \item[] Justification: The code for ReSIDe will not be released, as it is proprietary; however, complete instructions to reproduce the method are clearly outlined in the paper and appendices. The GenImage dataset is open-source and publicly available, and we have detailed in the paper and appendices how to prepare covariate-shifted test sets and hold-out validation sets based on the GenImage dataset.

\item {\bf Experimental setting/details}
    \item[] Question: Does the paper specify all the training and test details (e.g., data splits, hyperparameters, how they were chosen, type of optimizer) necessary to understand the results?
    \item[] Answer: \answerYes{} 
    \item[] Justification: Experimental settings/details are presented in Sections \ref{sec:eval_csf}, \ref{sec:experiments} and Appendix \ref{app:imp}.

\item {\bf Experiment statistical significance}
    \item[] Question: Does the paper report error bars suitably and correctly defined or other appropriate information about the statistical significance of the experiments?
    \item[] Answer: \answerYes{} 
    \item[] Justification: The results for ReSIDe are averaged over 3 trials, with standard deviations reported.

\item {\bf Experiments compute resources}
    \item[] Question: For each experiment, does the paper provide sufficient information on the computer resources (type of compute workers, memory, time of execution) needed to reproduce the experiments?
    \item[] Answer: \answerYes{} 
    \item[] Justification: Compute resource information is available in Appendix \ref{app:imp} and \ref{app:complex}.
    
\item {\bf Code of ethics}
    \item[] Question: Does the research conducted in the paper conform, in every respect, with the NeurIPS Code of Ethics \url{https://neurips.cc/public/EthicsGuidelines}?
    \item[] Answer: \answerYes{} 
    \item[] Justification: The research conforms to the NeurIPS Code of Ethics.

\item {\bf Broader impacts}
    \item[] Question: Does the paper discuss both potential positive societal impacts and negative societal impacts of the work performed?
    \item[] Answer: \answerYes{} 
    \item[] Justification: The work has no potential negative societal impacts. Potential positive societal impacts have been discussed in Section \ref{sec:intro}.
    
\item {\bf Safeguards}
    \item[] Question: Does the paper describe safeguards that have been put in place for responsible release of data or models that have a high risk for misuse (e.g., pre-trained language models, image generators, or scraped datasets)?
    \item[] Answer: \answerNA{} 
    \item[] Justification: We do not release data or models, and thus the paper poses no risks for misuse.

\item {\bf Licenses for existing assets}
    \item[] Question: Are the creators or original owners of assets (e.g., code, data, models), used in the paper, properly credited and are the license and terms of use explicitly mentioned and properly respected?
    \item[] Answer: \answerYes{} 
    \item[] Justification: We properly credit and cite all used existing assets and adhere to their terms of use.

\item {\bf New assets}
    \item[] Question: Are new assets introduced in the paper well documented and is the documentation provided alongside the assets?
    \item[] Answer: \answerNA{} 
    \item[] Justification: There are no released new assets.

\item {\bf Crowdsourcing and research with human subjects}
    \item[] Question: For crowdsourcing experiments and research with human subjects, does the paper include the full text of instructions given to participants and screenshots, if applicable, as well as details about compensation (if any)? 
    \item[] Answer: \answerNA{} 
    \item[] Justification: The paper does not involve crowdsourcing nor research with human subjects.

\item {\bf Institutional review board (IRB) approvals or equivalent for research with human subjects}
    \item[] Question: Does the paper describe potential risks incurred by study participants, whether such risks were disclosed to the subjects, and whether Institutional Review Board (IRB) approvals (or an equivalent approval/review based on the requirements of your country or institution) were obtained?
    \item[] Answer: \answerNA{} 
    \item[] Justification: The paper does not involve crowdsourcing nor research with human subjects.

\item {\bf Declaration of LLM usage}
    \item[] Question: Does the paper describe the usage of LLMs if it is an important, original, or non-standard component of the core methods in this research? Note that if the LLM is used only for writing, editing, or formatting purposes and does \emph{not} impact the core methodology, scientific rigor, or originality of the research, declaration is not required.
    \item[] Answer: \answerNA{} 
    \item[] Justification: The core method development in this research does not involve LLMs as any important, original, or non-standard components.

\end{enumerate}

\end{document}